\documentclass[twoside,11pt]{article}

\usepackage{jmlr2e}
\usepackage{graphicx}  
\usepackage{bm, hyperref,mathtools,amsmath,scrextend, array, setspace, algorithm, textcomp, amssymb, comment, times, subfig}
\newcolumntype{M}[1]{>{\centering\arraybackslash}m{#1}}

\newtheorem{prop}{Proposition}

\jmlrheading{1}{2020}{1-48}{4/00}{10/00}{meila00a}{Siddharth Roheda and Hamid Krim}

\ShortHeadings{Volterra Neural Networks (VNNs)}{Roheda and Krim}
\firstpageno{1}

\begin{document}

\title{Volterra Neural Networks (VNNs)}

\author{\name Siddharth Roheda \email sroheda@ncsu.edu \\
       \addr Electrical and Computer Engineering Department\\
       North Carolina State University\\
       Raleigh, NC 27606, USA
       \AND
       \name Hamid Krim \email ahk@ncsu.edu \\
       \addr Electrical and Computer Engineering Department\\
       North Carolina State University\\
       Raleigh, NC 27606, USA
   		\AND
   		\name Bo Jiang \email bjiang8@ncsu.edu \\
   		\addr Electrical and Computer Engineering Department\\
   		North Carolina State University\\
   		Raleigh, NC 27606, USA}

\editor{}

\maketitle

\begin{abstract}
The importance of inference in Machine Learning (ML) has led to an explosive number of different proposals, particularly in Deep Learning. In an attempt to reduce the complexity of Convolutional Neural Networks, we propose a Volterra filter-inspired Network architecture. This architecture introduces controlled non-linearities in the form of interactions between the delayed input samples of data. We propose a cascaded implementation of Volterra Filtering so as to significantly reduce the number of parameters required to carry out the same classification task as that of a conventional Neural Network. We demonstrate an efficient parallel implementation of this Volterra Neural Network (VNN), along with its remarkable performance while retaining a relatively simpler and potentially more tractable structure. Furthermore, we show a rather sophisticated adaptation of this network to nonlinearly fuse the RGB (spatial) information and the Optical Flow (temporal) information of a video sequence for action recognition. The proposed approach is evaluated on UCF-101 and HMDB-51 datasets for action recognition, and is shown to outperform state of the art CNN approaches. The code-base for our paper is available on request (Fill \href{https://forms.gle/JvmKHLKLMeWzKsuJ8}{THIS} form).
\end{abstract}


\section{Introduction}
Human action recognition is an important research topic in Computer Vision, and may be useful in surveillance, video retrieval, and man-machine interaction to name a few. The survey on Action Recognition approaches \cite{kong2018human} provides a good progress overview. Video classification usually involves three stages \cite{wang2009evaluation,liu2009recognizing,niebles2010modeling,sivic2003video,karpathy2014large}, namely, visual feature extraction (local features like Histograms of Oriented Gradients (HoG) \cite{dalal2005histograms}, or global features like Hue, Saturation, etc.), feature fusion/concatenation, and lastly classification. In \cite{yi2011human}, an intrinsic stochastic modeling of human activity
on a shape manifold is proposed and an accurate analysis of the non-linear feature space of activity models is provided. The emergence of Convolutional Neural Networks (CNNs) \cite{lecun1998gradient}, along with the availability of large training datasets and computational resources have come a long way to obtaining the various steps by a single neural network. This approach has led to remarkable progress in action recognition in video sequences, as well as in other vision applications like object detection \cite{sermanet2013overfeat}, scene labeling \cite{farabet2012learning}, image generation \cite{goodfellow2014generative}, image translation \cite{isola2017image}, information distillation \cite{roheda2018cross,hoffman2016learning}, etc. In the Action Recognition domain, datasets like the UCF-101 \cite{soomro2012ucf101}, Kinetics \cite{kay2017kinetics}, HMDB-51 \cite{kuehne2011hmdb}, and Sports-1M \cite{karpathy2014large} have served as benchmarks for evaluating various solution performances. In action recognition applications the proposed CNN solutions generally align along two themes: \textit{1. One Stream CNN} (only use either spatial or temporal information), \textit{2. Two Stream CNN} (integrate both spatial and temporal information). Many implementations \cite{carreira2017quo,diba2017deep,feichtenhofer2016convolutional,simonyan2014two} have shown that integrating both streams leads to a significant boost in recognition performance. In Deep Temporal Linear Encoding \cite{diba2017deep}, the authors propose to use 2D CNNs (pre-trained on ImageNet \cite{deng2009imagenet}) to extract features from RGB frames (spatial information) and the associated optical flow (temporal information). The video is first divided into smaller segments for feature extraction via 2D CNNs. The extracted features are subsequently combined into a single feature map via a bilinear model. This approach, when using both streams, is shown to achieve a \textit{95.6 \%} accuracy on the UCF-101 dataset, while only achieving \textit{86.3 \%} when only relying on the RGB stream. Carreira et al. \cite{carreira2017quo} adopt the GoogLeNet architecture which was developed for image classification in ImageNet \cite{deng2009imagenet}, and use 3D convolutions (instead of 2D ones) to classify videos. This implementation is referred to as the Inflated 3D CNN (I3D), and is shown to achieve a performance of \textit{88.8 \%} on UCF-101 when trained from scratch, while achieving a \textit{98.0 \%} accuracy when a larger dataset (Kinetics) was used for pre-training the entire network (except for the classification layer). While these approaches achieve near perfect classification, the models are extremely heavy to train, and have a tremendous number of parameters (25M in I3D, 22.6M in Deep Temporal Linear Encoding). This in addition, makes the analysis, including the necessary degree of non-linearity, difficult to understand, and the tractability elusive. In this paper we explore the idea of introducing controlled non-linearities through interactions between delayed samples of a time series.
We will build on the formulations of the widely known Volterra Series \cite{volterra2005theory} to accomplish this task. While prior attempts to introduce non-linearity based on the Volterra Filter have been proposed \cite{kumar2011trainable,zoumpourlis2017non}
, most have limited the development up to a quadratic form on account of the explosive number of parameters required to learn higher order complexity structure. While quadratic non-linearity is sufficient for some applications (eg. system identification), it is highly inadequate to capture all the non-linear information present in videos. 

\textbf{Contributions:} In this paper, we propose a Volterra Filter \cite{volterra2005theory} based architecture where the non-linearities are introduced via the system response functions and hence by controlled interactions between delayed frames of the video. The overall model is updated on the basis of a cross-entropy loss of the labels resulting from a linear classifier of the generated features. An efficiently cascaded implementation of a Volterra Filter is used in order to explore higher order terms while avoiding over-parameterization. The Volterra filter principle is also exploited to combine the RGB and the Optical Flow streams for action recognition, hence yielding a performance driven non-linear fusion of the two streams. We further show that the number of parameters required to realize such a model is significantly lower in comparison to a conventional CNN, hence leading to faster training and a significant reduction of the required resources to learn, store, or implement such a model.

\section{Background and Related Work}
\subsection{Volterra Series}
A dynamical system, viewed as a black box, is characterized by its input/output relationship $y_t/x_t$. If a non-linear system is time invariant, the input/output relation can be expressed in the following form \cite{volterra2005theory,schetzen1980volterra},
\begin{gather}
	\label{Volterra_Series}
	\begin{align}
		&y_t = \sum_{\tau_1=0}^{L-1} \bm{W}^{\bm{1}}_{\tau_1}x_{t-\tau_1} + \sum_{\tau_1,\tau_2=0}^{L-1}  \bm{W}^{\bm{2}}_{\tau_1, \tau_2}x_{t-\tau_1}x_{t-\tau_2} + ... + \sum_{\tau_1,\tau_2,...,\tau_K=0}^{L-1} \bm{W}^{\bm{K}}_{\tau_1,\tau_2,...,\tau_K}x_{t-\tau_1}x_{t-\tau_2}...x_{t-\tau_K},
	\end{align}
\end{gather}
where $x_t, y_t \in \mathbb{R}$ and $\bm{W}^{\bm{k}} \in \mathbb{R}^k$, $L$ is the number of terms in the filter memory (also referred to as the filter length), $\bm{W}^{\bm{k}}$ are the weights for the $k^{th}$ order term, and $\bm{W}^{\bm{k}}_{\tau_1,\tau_2,...,\tau_k} = 0$, for any $\tau_j < 0$, $j = 1,2,...,k$, $\forall k=1,2,...,K$ due to causality.
This functional form is due to the mathematician Vito Volterra \cite{volterra2005theory}, and is widely referred to as the Volterra Series.
\begin{figure}[h!]
	\centering
	\includegraphics[width=0.4\textwidth]{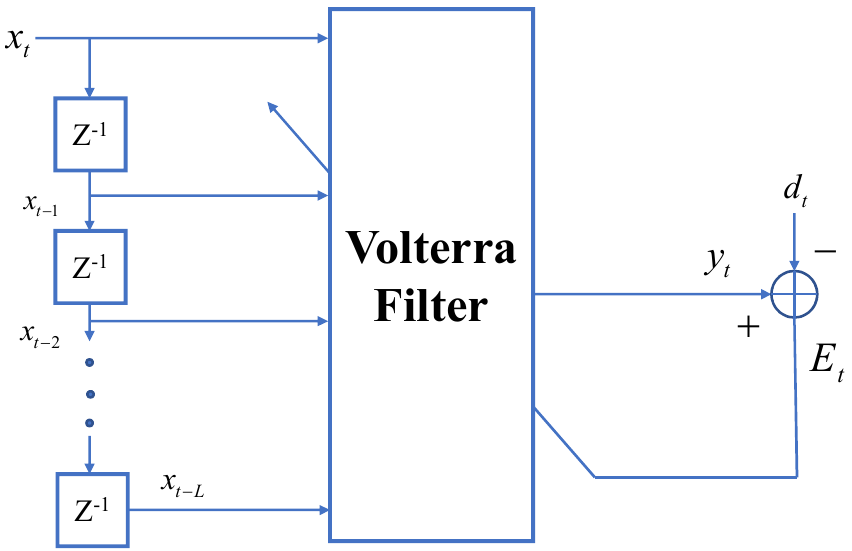}
	\caption{Adaptive Volterra Filter}
	\label{VF}
\end{figure}
The calculation of the kernels is computationally complex, and a $K^{th}$ order filter of length $L$, entails solving $L^K$ equations. The associated adaptive weights are a result of a target energy functional minimization iteratively adapting the filter taps as shown in Figure \ref{VF}. 

It is worth observing from Equation \ref{Volterra_Series} that the linear term is actually similar to a convolutional layer in CNNs. Non-linearities in CNNs are implicitly introduced via activation functions, following the convolutional layer, while those in Equation \ref{Volterra_Series} are introduced via higher order convolutions.

\subsection{Nested Volterra Filter}
A closer inspection of Equation \ref{Volterra_Series} immediately suggests its somewhat simplified form by way of nesting repeated terms \cite{osowski1994multilayer}, and yield, 
\begin{equation}
	\label{Nested_VF}
	\begin{split}
		y_t = \sum_{\tau_1 = 0}^{L-1} x_{t - \tau_1}\Biggl[ \bm{W}^{\bm{1}}_{\tau_1} + \sum_{\tau_2=0}^{L-1} x_{t - \tau_2} \biggl[ \bm{W}^{\bm{2}}_{\tau_1,\tau_2} + \sum_{\tau_3=0}^{L-1} x_{t-\tau_3}[ \bm{W}^{\bm{3}}_{\tau_1,\tau_2,\tau_3} + ... ]  \biggr] \Biggr].
	\end{split}
\end{equation}
Each factor contained within brackets can be interpreted as the output of a linear Finite Impulse Response (FIR) filter, thus allowing a layered representation of the Filter. A nested filter implementation with $L=2$ and $K=2$ is shown in Figure \ref{Nested_fig}. The length of the filter is increased by adding modules in parallel, while the order is increased by additional layers. Much like any multi-layer network, the weights of the synthesized filter are updated at successive layers according to a backpropogation scheme. 
\begin{figure}[h!]
	\centering
	\includegraphics[width=0.47\textwidth]{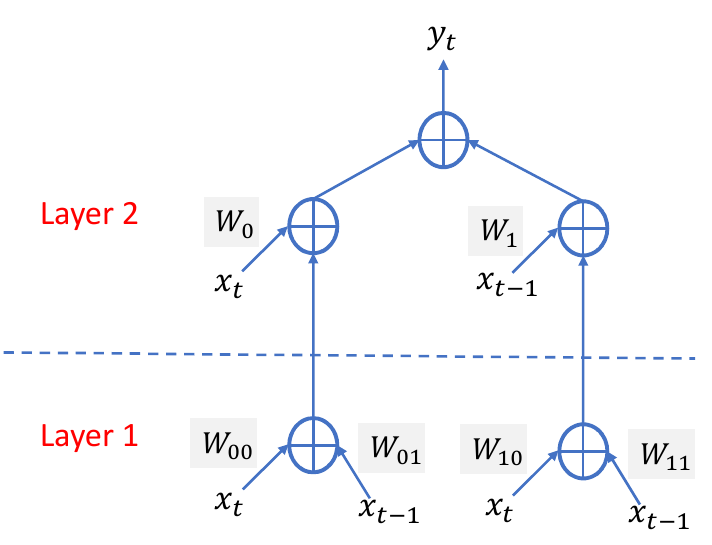}
	\caption{Nested Volterra Filter}
	\label{Nested_fig}
\end{figure}
The nested structure of the Volterra Filter, yields much faster learning in comparison to that based on Equation \ref{Volterra_Series}. It, however, does not reduce the number of parameters to be learned, leading to potential over-parameterization when learning higher order filter approximations. 
Such a structure was used for a system identification problem in \cite{osowski1994multilayer}. The mean square error between the desired signal ($d_t$) and the output of the filter ($y_t$) was used as the cost functional to be minimized,
\begin{equation}
	E_t = \frac{1}{2}(d_t - y_t)^2,
\end{equation}
and the weights for the $k^{th}$ layer are updated per the following equations,
\begin{gather}
	\begin{align}
		& \bm{W^k}_{\tau_1,\tau_2,...,\tau_k} (t+1) = \bm{W^k}_{\tau_1,\tau_2,...,\tau_k} (t) - \eta \frac{\partial E_t}{\partial \bm{W^k}_{\tau_1, \tau_2, ..., \tau_k}},\\
		&\frac{\partial E_t}{\partial \bm{W^k}_{\tau_1, \tau_2, ..., \tau_k}} = x_{t-\tau_k}x_{t-\tau_{k-1}}...x_{t-\tau_1}(y_t-d_t).
	\end{align}
\end{gather}

\subsection{Relation to Bilinear Convolution Neural Networks}
There has been work on introducing $2^{nd}$ order non-linearities in the network by using a bi-linear operation on the features extracted by convolutional networks. Bilinear-CNNs (B-CNNs) were introduced in \cite{lin2015bilinear} and were used for image classification. In B-CNNs, a bilinear mapping is applied to the final features extracted by linear convolutions leading to $2^{nd}$ order features which are not well localized. As a result a feature extracted from the lower right corner of a frame/image in the B-CNN case, may interact with a feature from the upper right corner, and this potential non-real interaction may introduce erroneous additional characteristics, in contrast to our proposed approach which highly controls such effects. Compact Bilinear Pooling was introduced in \cite{gao2016compact} where a bilinear approach to reduce feature dimensions was introduced. This was again performed after all the features had been extracted via linear convolutions and was limited to quadratic non-linearities. In our work we will explore non-linearities of much higher orders and also account for continuity of information between video frames over a given time interval with the immediately preceding period. This effectively achieves a Recurrent Network-like property which accounts for a temporal relationship. 

\subsection{Relation to Long-Short Term Memory}
\label{LSTM_rel_work}
Long-Short Term Memory Networks (LSTMs) \cite{hochreiter1997long} have been widely used to capture the long-term trajectory information in temporally evolving data. The sequential modeling
ability of LSTMs makes them particularly appealing for capturing long-range temporal dynamics in videos. An LSTM computes a mapping from an input sequence, $x = \{x_1,...,x_T\}$ to an output sequence $h = \{h_1,...,h_T\}$. 
The mapping of the input at time $t$, $x_t$ to the output, $h_t$ entails the following:
\begin{enumerate}
	\item \textbf{Forget Gate: }
	\begin{equation}
		\label{forget_gate}
		f_t = \sigma(W_{f}[h_{t-1}, x_t]) = \sigma(W_f^1 \cdot h_{t-1} + W_f^2 \cdot {x_t}).
	\end{equation}
	\item \textbf{Input Gate: }
	\begin{gather}
		\begin{align}
			& i_t = \sigma(W_{i}[h_{t-1}, x_t]) = \sigma(W_i^1 \cdot h_{t-1} + W_i^2 \cdot {x_t}),\\
			& \tilde{C_t} = \text{tanh}(W_c[h_{t-1}, x_t]),\\
			& C_t = f_t \cdot C_{t-1} + i_t \cdot \tilde{C_t}.	\label{input_gate}	
		\end{align}
	\end{gather}
	\item \textbf{Output Gate: }
	\begin{gather}
		\begin{align}
			&  o_t = \sigma(W_o[h_{t-1},x_t]) = \sigma(W_o^1 \cdot h_{t-1} + W_o^2 \cdot {x_t}),\\
			& h_t = o_t \cdot \text{tanh}(C_t). \label{h_t}
		\end{align}
	\end{gather}
\end{enumerate}
Where $\sigma$ is the sigmoid activation function, $\text{tanh}$ is the hyperbolic tangent function, and $W_f, W_i, W_c,$ and $W_o$ are the weight matrices characterizing the forget gate, input gate, cell state, and output gate respectively.
\section{Problem Statement}
Let a set of activities $\mathcal{A} = \{a_1,...,a_I\}$, be of interest following an observed sequence of frames $\bm{X}_{T \times H \times W}$, where $T$ is the total number of frames, and $H$ and $W$ are the height and width of a frame. At time $t$, the features $\bm{F}_t = g(\bm{X}_{[t-L+1:t]})$, will be used for classification of the sequence of frames $\bm{X}_{[t-L+1:t]}$ and mapped into one of the actions in $\mathcal{A}$, where $L$ is the number of frames in the memory of the system/process. A linear classifier with weights $\bm{W}^{cl} = \{\bm{w}^{cl}_i\}_{i = 1,...,I}$, and biases $\bm{b}^{cl} = \{ b_i^{cl} \}_{i = 1,...,I}$ will then be central to determining the classification scores for each activity, followed by a soft-max function ($\rho(.)$) to convert the scores into a probability measure. The probability that the sequence of frames be associated to the $i^{th}$ action class is hence the result of,
\begin{equation}
	P_t(a_i) =\rho(\bm{w}_i^{{cl}^T} \cdot \bm{F}_{t} + b_i^{cl}) = \frac{exp(\bm{w}_i^{{cl}^T}.\bm{F}_{t} + b_i^{cl})}{\sum_{m=1}^{I} exp(\bm{w}_m^{{cl}^T}.\bm{F}_{t} + b_m^{cl})}.
\end{equation}

\section{Proposed Solution}
\subsection{Volterra Filter based Classification}
In our approach we propose a Volterra Filter structure to approximate a function $g(\cdot)$. Given that video data is of interest here, a spatio-temporal Volterra Filter must be applied. 
As a result, this 3D version of the Volterra Filter discussed in Section 2 is used to extract the features,
\begin{gather}
	\label{3dvf}
	\begin{align} 
		&\nonumber \bm{F}_{\left[\substack{t\\ s_1\\ s_2}\right]} = g\left(\bm{X}_{\left[ \substack{t-L+1:t \\ s_1 - p_1 : s_1 + p_1\\ s_2 - p_2 : s_2 + p_2} \right]}\right) = \sum_{\tau_1, \sigma_{11}, \sigma_{21}} \bm{W}^{\bm{1}}_{\left[\substack{\tau_1 \\\sigma_{11} \\\sigma_{21}}\right]} x_{\left[\substack{t - \tau_1 \\s_1 - \sigma_{11} \\s_2 - \sigma_{21}}\right]} \\
		& + \sum_{\substack{\tau_1, \sigma_{11}, \sigma_{21}\\ \tau_2, \sigma_{12}, \sigma_{22}}} \bm{W}^{\bm{2}}_{\left[\substack{\tau_1 \\\sigma_{11} \\\sigma_{21}}\right] \left[\substack{\tau_2 \\\sigma_{12} \\\sigma_{22}}\right]} x_{\left[\substack{t - \tau_1 \\s_1 - \sigma_{11} \\s_2 - \sigma_{21}}\right]} x_{\left[\substack{t - \tau_2 \\s_1 - \sigma_{12} \\s_2 - \sigma_{22}}\right]} + ...
	\end{align}
\end{gather}
\begin{figure*}[tbp]
	\centering
	\includegraphics[width = 0.95\textwidth]{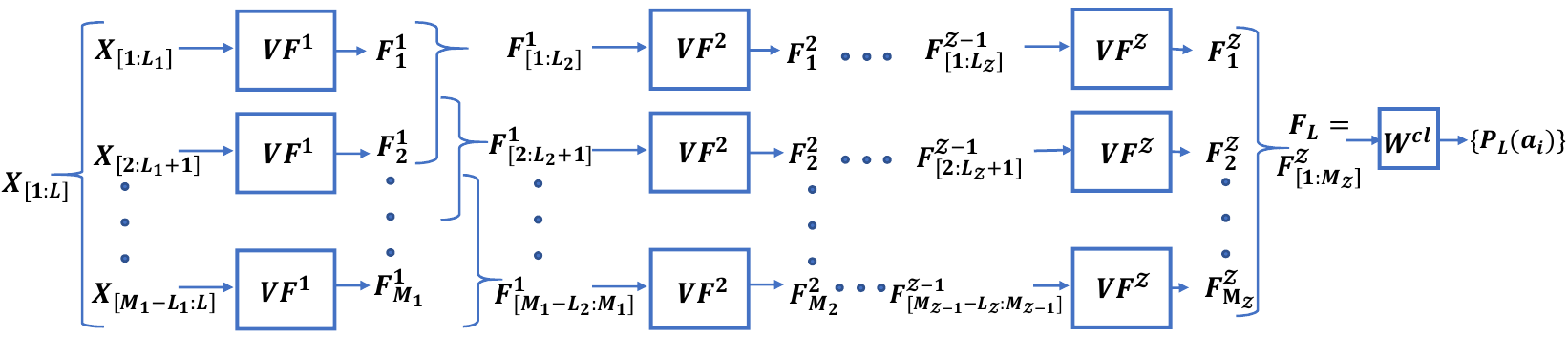}
	\centering
	\caption{Block diagram for an Overlapping Volterra Neural Network}
	\label{bd_vf}
\end{figure*}
where, $\tau_j \in [0,L-1]$, $\sigma_{1j} \in [-p_1,p_1]$, and $\sigma_{2j} \in [-p_2, p_2]$ respectively represent the temporal and spatial translations (horizontal and vertical directions).
Following this formulation, and as discussed in Section 3, a linear classifier is used to determine the probability of each activity in $\mathcal{A}$. Updating the filter parameters is pursued by minimizing some measure of discrepancy relative to the ground truth and the probability determined by the model. Our adopted measure herein is the cross-entropy loss computed as, 
\begin{equation}
	E = \sum_{t,I} -d_{t_i} log P_{t}(a_i),
\end{equation}
where, $t \in \{1,L+1,2L+1,...,T\}$, $i \in \{1,2,...,I\}$, and $d_{t_i}$ is the ground truth label for $\bm{X}_{[t-L+1:t]}$ belonging to the $i^{th}$ action class.
In addition to minimizing the error, we also include a weight decay in order to ensure generalizability of the model by penalizing large weights. So, the overall cost functional which serves as a target metric is written as,
\begin{equation}
	\begin{split}
		\min_{g} \sum_{t,I} -d_{t_i} \log \rho(\bm{w}_i^{cl^{T}} \cdot g(\bm{X}_{[t-L+1:t]}) + b_i^{cl}) + \frac{\lambda}{2} \left[ \sum_{k=1}^{K} \left\| \bm{W^k} \right\|_2^2 + \left\| \bm{W}^{cl} \right\|_2^2 \right],
	\end{split}
\end{equation}
where $\rho$ is the soft-max function, and $K$ is the order of the filter. 

\begin{prop}
	A VNN architecture provides any continuous function over a compact $\Omega \subset  {\mathbb R}^d$ (including so-called activation functions) an approximation up to an error margin defined by the Taylor Remainder Theorem.
\end{prop}

\begin{proof}
	Based on the Weistrass Approximation Theorem \cite{stone1948generalized} it is known that any continuous non-linear function can be approximated using a polynomial. Specifically, the taylor expansion of the non-linear function may be used (i.e. a generic so-called activation function) for $x\in \Omega\subset {\mathbb R}$,
	\begin{equation}
		\sigma(x) = c^0 + c^1x + c^2x^2 + ... + c^kx^k + ... + c^{\infty}x^{\infty}. 
	\end{equation}
	In particular, a sigmoid activation can be approximated as,
	\begin{equation}
		\sigma_{sigmoid}(x) = \frac{1}{1+e^{-x}} = \frac{1}{2} + \frac{1}{4}x - \frac{1}{48}x^3 + \frac{1}{480} x^5 + ... 
		\label{sigmoid}
	\end{equation}
	As seen from Equation \ref{Volterra_Series} the VNN formulation can specifically learn such an expansion up to a finite order,
	\begin{equation}
		\sigma_{VNN} (x) = w_0 + w_1x + w_2x^2 + ... + w_kx^k,
		\label{volterra_k}
	\end{equation}
	which is a $k^{th}$ order approximation of $\sigma(x)$. Here $w_k$ is the $k^{th}$ order weight and is learned during the training process. In comparison to the coefficients in Equation \ref{sigmoid}, we get an approximation of the sigmoid activation function.
	
	A finite order polynomial expansion yields an error which can be expressed via the Taylor Remainder Theorem \cite{firey1960remainder},
	\begin{equation}
		\lvert \sigma(x) - \sigma_{VNN}(x) \rvert \leq R_k = \left\lvert \frac{\sigma^{k+1}(m)}{(k+1)!} (x-a)^{k+1} \right \rvert
		\label{taylor_remainder_thm}
	\end{equation}  
	where the Taylor Expansion is centered around $a$ and $m$ lies between $a$ and $x$.
	
\end{proof}

Figure \ref{Approximations} shows the approximation of various activation functions via the Volterra series formulation. Note that the $w_k's$ in Equation \ref{volterra_k} are actually learnable weights and are updated on the basis of the classification error when comparing the classifier output with the ground truth. N.B. The coefficients in Equation \ref{sigmoid} are for a generic "sigmoid" function, and are to be distinguished from the learnable coefficients when using it as a local approximation of the data.

\begin{figure}
	\centering
	\includegraphics[width=0.65\textwidth]{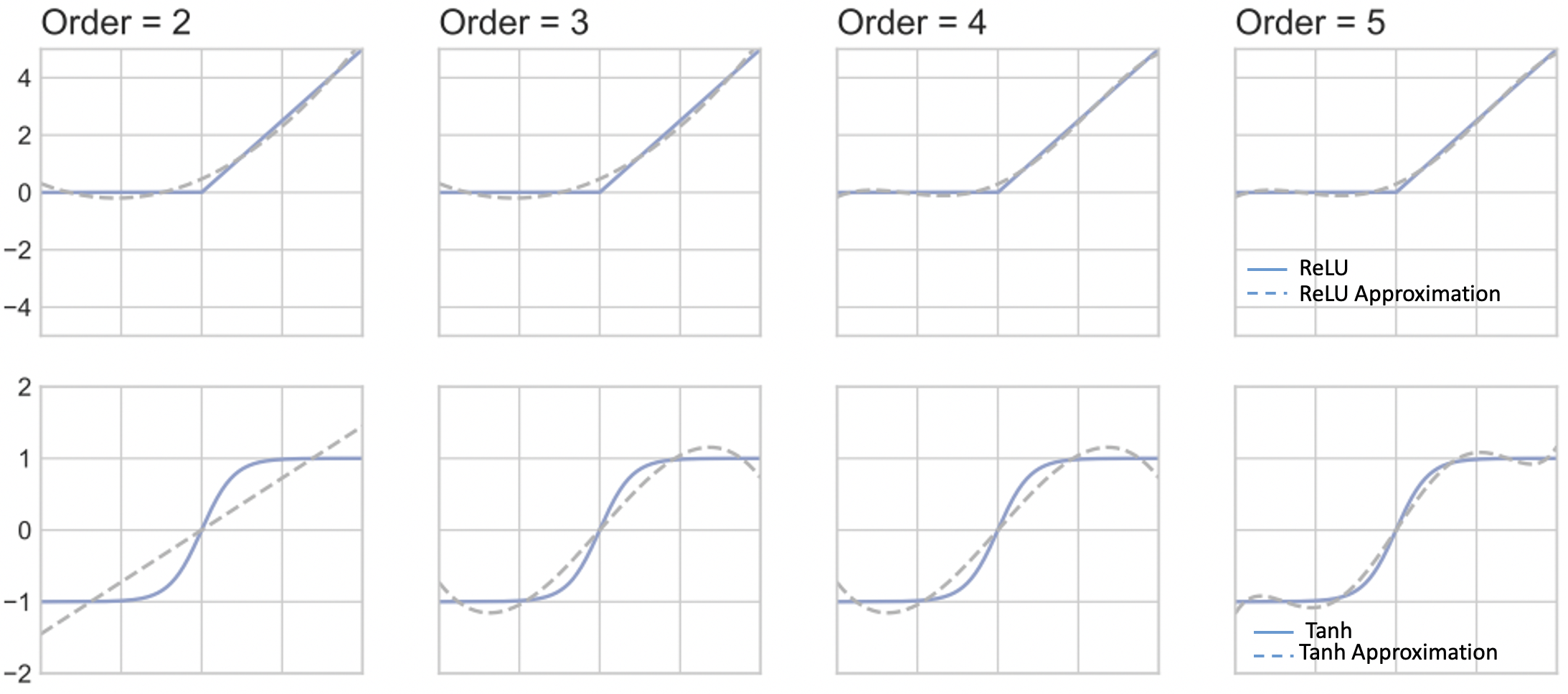}
	\caption{Approximation of ReLU and Tanh activation via the Volterra series formulation}
	\label{Approximations}
\end{figure}


\begin{prop}
	A VNN approximation independently achieves optimal weighting of higher order moments and cross-moments of data, for a more adaptive representation than that possible by activation functions.
\end{prop}
\begin{proof}
	Consider an input $X = [x_1, x_2]$. A standard single layer neural network would take the dot product of $X$ with the linear weights $W = [w_1, w_2]$ and apply an activation function $\sigma(.)$ to get the output $y = \sigma(W^T.X) = \sigma(w_1x_1 + w_2x_2)$.
	
	As discussed in Proposition 1, the activation function (eg. ReLU, sigmoid, tanh) $\sigma(.)$ can be approximated by its Taylor series expansion. Considering a $2^{nd}$ order approxmation,
	
	\begin{gather}
		\begin{align}
			& \nonumber \text{ } y = c_0 + c_1 (w_1x_1 + w_2x_2) + c_2 (w_1x_1 + w_2x_2)^2\\
			& \quad = c_0 + c_1w_1x_1+c_1w_2x_2 + c_2w_1^2x_1^2 + c_2w_2^2x_2^2 + 2c_2w_1w_2x_1x_2.
		\end{align}
	\end{gather} 
	This can be rewritten as,
	\begin{equation}
		y = \alpha_0 + \alpha_1x_1 + \alpha_2x_2 + \alpha_3x_1^2 + \alpha_4x_2^2 + \alpha_5x_1x_2.
		\label{alpha_eq}
	\end{equation}
	From Equation 20 and \ref{alpha_eq}, 
	\begin{gather}
		\begin{align}
			&\alpha_0 = c_0 \text{; }  \alpha_1 = c_1w_1 \text{; }\\
			&\alpha_2 = c_1w_2 \text{; } \alpha_3 = c_2w_1^2 \text{; }\\
			&\alpha_4 = c_2w_2^2 \text{; } \alpha_5 = c_1c_2w_1w_2.
		\end{align}
	\end{gather}
	As a result,
	\begin{equation}
		y = \alpha_0 + \alpha_1x_1 + \alpha_2x_2 + \frac{c_2}{c_1}\alpha_1^2x_1^2 + \frac{1}{c_2}\alpha_2^2x_2^2 + \frac{2}{c_1}\alpha_1\alpha_2x_1x_2.
	\end{equation}
	
	It is clear that the moments (monomials) and cross-moments (cross-products) weights are computationally tied.  This unavoidable and implicit constraint among the moments and cross-moments make  the adaptive approximation of such features unlikely. This clearly scales to higher oder approximation highlighting the flexibility of VNN in independently adapting the associated weights.
	
	The additional relative feature preservation limitation (at different scales) of activation functions  is reflected in the allocated weighting of higher order according to the  Taylor series evolution. For an $n^{th}$ order taylor approximation we have, 
	\begin{equation}
		\sigma(x) = \sum_{n=0}^{N} \frac{\sigma^{(n)}(a)}{n!} (x-a)^n,
	\end{equation}
	where $\sigma^{(n)}(a)$ is the $n^{th}$ derivative of $\sigma$ at $a$. As seen from this equation the $n^{th}$ coefficient, $c_n$ is given as $c_n = \frac{\sigma^{(n)}(a)}{n!}$. As a result of the $n!$ in the denominator, $c_{n+1} < c_{n} \text{ } \forall n$. This leads to decreasing importance of higher order features, and makes it difficult for the model to learn them with no regards to their potentially discriminative role in inference.
\end{proof}

The Volterra series formulation, in contrast, assigns an independent learnable weight to each term in Equation \ref{alpha_eq}.
\subsection{Cascaded Volterra Filters: Managing High Order Complexity}
A major challenge in learning the afore-mentioned architecture arises when higher order non-linearities are sought. The number of required parameters for a $K^{th}$ order filter is, 
\begin{equation}
	\sum_{k=1}^K (L \cdot [2p_1+1] \cdot [2p_2+1])^k. 
	\label{num_param_Kth}
\end{equation}
This complexity increases exponentially when the order is increased, thus making a higher order ($>3$) Volterra Filter implementation impractical. To alleviate this limitation, we use a cascade of $2^{nd}$ order Volterra Filters, wherein, the second order filter is repeatedly applied until the desired order is attained. 

A $K^{th}$ order filter is realized by applying the $2^{nd}$ order filter $\mathcal{Z}$ times, where, $K = 2^{2^{(\mathcal{Z}-1)}} $. 
If the length of the first filter in the cascade is $L_1$, the input video $\bm{X}_{[t-L+1:t]}$ can be viewed as a concatenation of a set of shorter videos,
\begin{equation}
	\begin{split}
		\bm{X}_{[t_L:t]} = \biggl[ \bm{X}_{[t_L:t_L+L_1]} \text{ } \bm{X}_{[t_L+L_1:t_L+2L_1]} ... \bm{X}_{[t_L+(M_1-1)L_1:t_L+M_1L_1]} \biggr],
	\end{split}
\end{equation}
where $M_1 = \frac{L}{L_1}$, and $t_L = t-L+1$. Now, a $2^{nd}$ order filter $g_1(.)$ applied on each of the sub-videos leads to the features,
\begin{equation}
	\begin{split}
		\bm{F}^1_{t_{[1:M_1]}} = \biggl[ g_1(\bm{X}_{[t_L:t_L+L_1]}) \quad g_1( \bm{X}_{[t_L+L_1:t_L+2L_1]}) ... g_1( \bm{X}_{[t_L+(M_1-1)L_1:t_L+M_1L_1]})  \biggr].
	\end{split}
\end{equation}
A second filter $g_2(.)$ of length $L_2$ is then applied to the output of the first filter,
\begin{equation}
	\begin{split}
		\bm{F}^2_{t_{[1:M_2]}} = \biggl[ g_2( \bm{F}^1_{t_{[1:L_2]}} ) \quad g_2( \bm{F}^1_{t_{[L_2+1:2L_2]}}) ... g_2( \bm{F}^1_{t_{[(M_2-1)L_2+1:(M_2L_2)]}})  \biggr],
	\end{split}
\end{equation}
where, $M_2 = \frac{M_1}{L_2}$. Note that the features in the second layer are generated by taking quadratic interactions between those generated by the first layer, hence, leading to $4^{th}$ order terms.

Finally, for a cascade of $\mathcal{Z}$ filters, the final set of features is obtained as,
\begin{equation}
	\begin{split}
		\bm{F}^\mathcal{Z}_{t_{[1:M_\mathcal{Z}]}} = \biggl[ g_\mathcal{Z}( \bm{F}^{\mathcal{Z}-1}_{t_{[1:L_\mathcal{Z}]}} ) \quad g_\mathcal{Z}( \bm{F}^{\mathcal{Z}-1}_{t_{[L_\mathcal{Z}+1:2L_\mathcal{Z}]}}) ... g_\mathcal{Z}( \bm{F}^{\mathcal{Z}-1}_{t_{[(M_\mathcal{Z}-1)L_\mathcal{Z}+1:(M_\mathcal{Z}L_\mathcal{Z})]}})  \biggr],
	\end{split}
\end{equation}
where, $M_\mathcal{Z} = \frac{M_{\mathcal{Z}-1}}{L_\mathcal{Z}}$.

Note that these filters can also be implemented in an overlapped structure yielding the following features for the $z^{th}$ layer, $z \in \{ 1,...,\mathcal{Z} \}$,
\begin{equation}
	\begin{split}
		\bm{F}^z_{t_{[1:M_z]}} = \biggl[ g_z( \bm{F}^{z-1}_{t_{[1:L_z]}} ) \quad g_z( \bm{F}^{z-1}_{t_{[2:L_z+1]}}) ... g_z( \bm{F}^{z-1}_{t_{[(M_{z-1}) - L_z+1:M_{z-1}]}})  \biggr],
	\end{split}
\end{equation}
where $M_z = M_{z-1} - L_z + 1$.
The implementation of an Overlapping Volterra Neural Network (O-VNN) to find the corresponding feature maps, for an input video is shown in Figure \ref{bd_vf}.

\begin{prop}
	\label{VNN_order_prop}
	If $\mathcal{Z}$ $2^{nd}$ order filters are cascaded following the noted overlapped structure (as in Figure \ref{bd_vf}), the resulting Volterra Network has an effective order of $K_{\mathcal{Z}} = 2^{2^{\mathcal{Z}-1}}$.
\end{prop}

\begin{proof}
	Since each layer of the O-VNN is a $2^{nd}$ order Volterra Filter, the order at the $\mathcal{Z}^{th}$ layer can be written in terms of the order of the previous layer,
	\begin{equation}
		K_{\mathcal{Z}} = K_{\mathcal{Z}-1}^2,
		\label{rel_prev_lyr}
	\end{equation}
	where, $K_{\mathcal{Z}-1}$ is the order of the $(\mathcal{Z}-1)^{th}$ layer. 
	Since, the O-VNN only consists of $2^{nd}$ order layers, there exists some $p$ such that,
	\begin{equation}
		K_{\mathcal{Z}} = 2^p.
		\label{2^p}
	\end{equation}
	From Equations (\ref{rel_prev_lyr}) and (\ref{2^p}), 
	\begin{equation}
		2^p = K_{\mathcal{Z}-1}^2
	\end{equation}
	Taking $\log_2$ on both sides,
	\begin{gather}
		\begin{align}
			&\nonumber \log_2 2^p = \log_2 K_{\mathcal{Z}-1}^2\\
			& \nonumber \implies p = 2 \log_2 K_{\mathcal{Z}-1}\\
			& \nonumber \implies p = 2 \log_2 K_{\mathcal{Z}-2}^2\\
			& \nonumber \implies p = 2^2 \log_2 K_{\mathcal{Z}-2}\\
			& \implies p = 2^{(\mathcal{Z}-1)} \log_2 K_1
		\end{align}
	\end{gather}
	Since $K_1 = 2$ and $\log_2 2 =1$,
	\begin{equation}
		p = 2^{\mathcal{Z}-1}.
	\end{equation}
	Putting this in Equation (\ref{2^p}), we get,
	\begin{equation}
		K_{\mathcal{Z}} = 2^{2^{\mathcal{Z}-1}}.
	\end{equation}
\end{proof}

A natural question which arises about the resulting relative complexity of our proposed strategy, is addressed next.
\begin{prop}
	The complexity of a $K^{th}$ order cascaded Volterra filter will consist of,
	\begin{gather}
		\begin{align}
			& \sum_{z=1}^{\mathcal{Z}}\biggl[ (L_z \cdot [2p_{1_z}+1] \cdot [2p_{2_z}+1])  + (L_z \cdot [2p_{1_z}+1] \cdot [2p_{2_z}+1])^2 \biggr]
		\end{align}
	\end{gather}
	parameters.
\end{prop}
\begin{proof}
	For a $2^{nd}$ order filter ($K=2$), the number of parameters required is $\biggl[(L \cdot [2p_1+1] \cdot [2p_2+1]) + (L \cdot [2p_1+1] \cdot [2p_2+1])^2\biggr]$ (from equation \ref{num_param_Kth}). When such a filter is repeatedly applied $\mathcal{Z}$ times, it will lead to $\sum_{z=1}^{\mathcal{Z}}\biggl[ (L_z \cdot [2p_{1_z}+1] \cdot [2p_{2_z}+1])+ (L_z \cdot [2p_{1_z}+1] \cdot [2p_{2_z}+1])^2 \biggr]$ parameters with order $K=2^{2^{(\mathcal{Z}-1)}}$.
\end{proof}
Furthermore, if a multi-channel input/output is considered, the number of parameters is,
\begin{gather}
	\begin{align}
		& \sum_{z=1}^{\mathcal{Z}} (n_{ch}^{z-1} \cdot n_{ch}^z) \biggl[ (L_z \cdot [2p_{1_z}+1] \cdot [2p_{2_z}+1]) + (L_z \cdot [2p_{1_z}+1] \cdot [2p_{2_z}+1])^2 \biggr],
	\end{align}
\end{gather}
where $n_{ch}^z$ is the number of channels in the output of the $z^{th}$ layer.

\subsection{System Stability and Convergence}
The discussed system can be shown to be stable when the input is bounded, i.e. the system is Bounded Input Bounded Output (BIBO) stable.  
\begin{prop}
	An O-VNN with $\mathcal{Z}$ layers is BIBO stable if $\forall z \in \{1,...,\mathcal{Z}\}$,
	\begin{equation}
		\sum_{\tau_1, \sigma_{11}, \sigma_{21}} \left|\bm{W^{z1}}_{\left[ \substack{\tau_1\\ \sigma_{11}\\ \sigma_{21}} \right]} \right| + \sum_{\substack{\tau_1, \sigma_{11}, \sigma_{21}\\ \tau_2, \sigma_{12}, \sigma_{22}}} \left| \bm{W^{z2}}_{\left[ \substack{\tau_1\\ \sigma_{11}\\ \sigma_{21}} \right] \left[ \substack{\tau_2\\ \sigma_{12}\\ \sigma_{22}} \right] } \right| < \infty.
		\label{sufficient_cond1}
	\end{equation}
\end{prop}

\begin{proof}
	Consider the $z^{th}$ layer in the Cascaded implementation of the Volterra Filter,
	\begin{equation}
		\begin{split}
			\bm{F}^z_{{[1:M_z]}} = \biggl[ g_z( \bm{F}^{z-1}_{t_{[1:L_z]}} ) \quad g_z( \bm{F}^{z-1}_{t_{[2:L_z+1]}}) ... g_z( \bm{F}^{z-1}_{t_{[(M_{z-1}) - L_z +1:(M_{z-1})]}})  \biggr],	
		\end{split}
	\end{equation} 
	where, $M_z = M_{z-1} - L_z + 1$. Then for $m_z \in \{1,...,M_z\}$,
	\begin{gather}
		\begin{align}
			&\left|\bm{F}^z_{{\left[ \substack{m_z\\ s_1\\ s_2} \right]}}\right| = \left|g_z\left(\bm{F}^{z-1}_{{\left[ \substack{m_{z-1}-L_z+1:m_z\\ s_1-p_1:s_1+p1\\ s_2-p_2:s_2:p_2} \right]}}\right)\right| \\
			&\nonumber = \Biggl| \sum_{\tau_1, \sigma_{11}, \sigma_{21}} \bm{W^{z1}}_{\left[ \substack{\tau_1\\ \sigma_{11}\\ \sigma_{21}} \right]}f^{z-1}_{\left[ \substack{(L_z + m_z - 1)-\tau_1\\ s_1 - \sigma_{11}\\ s_2 - \sigma_{21}} \right]}\\
			& + \sum_{\substack{\tau_1, \sigma_{11}, \sigma_{21}\\ \tau_2, \sigma_{12}, \sigma_{22}}} \bm{W^{z2}}_{\left[ \substack{\tau_1\\ \sigma_{11}\\ \sigma_{21}} \right] \left[ \substack{\tau_2\\ \sigma_{12}\\ \sigma_{22}} \right] }f^{z-1}_{\left[ \substack{(L_z + m_z - 1)-\tau_1\\ s_1 - \sigma_{11}\\ s_2 - \sigma_{21}} \right]} f^{z-1}_{\left[ \substack{(L_z + m_z - 1)-\tau_2\\ s_1 - \sigma_{12}\\ s_2 - \sigma_{22}} \right]} \Biggr|\\
			& \nonumber \leq  \sum_{\tau_1, \sigma_{11}, \sigma_{21}} \left|\bm{W^{z1}}_{\left[ \substack{\tau_1\\ \sigma_{11}\\ \sigma_{21}} \right]} \right| \left|f^{z-1}_{\left[ \substack{(L_z + m_z - 1)-\tau_1\\ s_1 - \sigma_{11}\\ s_2 - \sigma_{21}} \right]}\right|\\
			& + \sum_{\substack{\tau_1, \tau_2\\ \sigma_{11}, \sigma_{12}\\ \sigma_{21}, \sigma_{22}}} \left| \bm{W^{z2}}_{\left[ \substack{\tau_1\\ \sigma_{11}\\ \sigma_{21}} \right] \left[ \substack{\tau_2\\ \sigma_{12}\\ \sigma_{22}} \right] } \right| \left| f^{z-1}_{\left[ \substack{(L_z + m_z - 1)-\tau_1\\ s_1 - \sigma_{11}\\ s_2 - \sigma_{21}} \right]}\right| \left|f^{z-1}_{\left[ \substack{(L_z + m_z - 1)-\tau_2\\ s_1 - \sigma_{12}\\ s_2 - \sigma_{22}} \right]}\right|\\
			& \label{BIBO} \leq A \sum_{\tau_1, \sigma_{11}, \sigma_{21}} \left|\bm{W^{z1}}_{\left[ \substack{\tau_1\\ \sigma_{11}\\ \sigma_{21}} \right]} \right| + A^2 \sum_{\substack{\tau_1, \sigma_{11}, \sigma_{21}\\ \tau_2, \sigma_{12}, \sigma_{22}}} \left| \bm{W^{z2}}_{\left[ \substack{\tau_1\\ \sigma_{11}\\ \sigma_{21}} \right] \left[ \substack{\tau_2\\ \sigma_{12}\\ \sigma_{22}} \right] } \right|.
		\end{align}
	\end{gather}
	Note that Equation \ref{BIBO} states that a bounded input yields $\sum_{\tau_1, \sigma_{11}, \sigma_{21}} \left|f^{z-1}_{\left[ \substack{(L_z + m_z - 1)-\tau_1\\ s_1 - \sigma_{11}\\ s_2 - \sigma_{21}} \right]}\right|  \leq A$, for some $A<\infty$. Hence, the sufficient condition for the system to be BIBO stable is, 
	\begin{equation}
		\sum_{\tau_1, \sigma_{11}, \sigma_{21}} \left|\bm{W^{z1}}_{\left[ \substack{\tau_1\\ \sigma_{11}\\ \sigma_{21}} \right]} \right| + \sum_{\substack{\tau_1, \sigma_{11}, \sigma_{21}\\ \tau_2, \sigma_{12}, \sigma_{22}}} \left| \bm{W^{z2}}_{\left[ \substack{\tau_1\\ \sigma_{11}\\ \sigma_{21}} \right] \left[ \substack{\tau_2\\ \sigma_{12}\\ \sigma_{22}} \right] } \right| < \infty.
		\label{sufficient_cond}
	\end{equation}
	If the input data (i.e. video frames) is bounded, so is the output of each layer provided that Equation \ref{sufficient_cond} is satisfied $\forall z \in \{1,...,Z\}$, making the entire system BIBO stable. 
\end{proof}

\begin{prop}
	\label{convergence}
	An O-VNN is stable and convergent under the condition $|x_t| < \rho < \infty$,
	where $x_t$ is the input to the filter, and $\rho$ is the radius of convergence for the proposed Volterra Filter.
\end{prop}

\begin{proof}
	A Volterra Filter can be viewed as a power series,
	\begin{equation}
		y_t = \sum_{k=1}^K g_k[ax_t] = \sum_{k=1}^K a^k g_k[x_t],
		\label{Volterra_power_series}
	\end{equation}
	where $a$ is an amplification factor and, 
	\begin{equation}
		g_k[x_t] =  \sum_{\tau_1,...,\tau_k} \bm{W^k}_{[\tau_1,...,\tau_k]} x_{t-\tau_1}x_{t-\tau_2}...x_{t-\tau_k}.
	\end{equation}
	
	In general, for a power series $\sum_{k=1}^\infty c_k x^k$, converges only for $|x| < \rho$, where $\rho = (lim_{k \to \infty}$ $sup |c_k|^{1/k})^{-1}$ \cite{rudin1964principles}. 
	Setting $a=1$ in Equation \ref{Volterra_power_series}, and replacing the coefficients $c_k$ with the $k^{th}$ order Volterra Kernel $\bm{W^k}$,
	\begin{equation}
		\rho = (lim_{k \to \infty} sup |\bm{W^k}|^{1/k})^{-1}.
	\end{equation}
	Furthermore, since the system must also satisfy the BIBO stablity condition,
	\begin{equation}
		|x_t| < (lim_{k \to \infty} sup |\bm{W^k}|^{1/k})^{-1} < \infty.
	\end{equation}
\end{proof}

\subsection{Synthesis and Efficient Implementation of Volterra Kernels}
As noted earlier, the linear kernel ($1^{st}$ order) of the Volterra filter is similar to the convolutional layer in the conventional CNNs. As a result, it can easily be implemented using the 3D convolution function in Tensorflow \cite{abadi2016tensorflow}. The second order kernel may be approximated as a product of two 3-dimensional matrices (i.e. a seperable operator), 
\begin{equation}
	\bm{W^2}_{L \times P_1 \times P_2 \times L \times P_1 \times P_2} = \sum_{q=1}^{Q} \bm{W^{2}}_{a_{q_{L \times P_1 \times P_2 \times 1}}} \bm{W^{2}}_{b_{q_{1 \times L \times P_1 \times P_2}}},
	\label{approx}
\end{equation}
where, $P_1=2p_1+1$, and $P_2=2p_2+1$, and $p_1, p_2$ specify the spatial translations (horizontal and vertical). Accounting for Equation \ref{3dvf} yields,
\begin{gather}
	\begin{align} 
		&\nonumber g\left(\bm{X}_{\left[ \substack{t-L+1:t \\ s_1 - p_1 : s_1 + p_1\\ s_2 - p_2 : s_2 + p_2} \right]}\right) = \sum_{\tau_1, \sigma_{11}, \sigma_{21}} \bm{W}^{\bm{1}}_{\left[\substack{\tau_1 \\\sigma_{11} \\\sigma_{21}}\right]} x_{\left[\substack{t - \tau_1 \\s_1 - \sigma_{11} \\s_2 - \sigma_{21}}\right]} \\
		& + \sum_{\substack{\tau_1, \sigma_{11}, \sigma_{21}\\ \tau_2, \sigma_{12}, \sigma_{22}}} \sum_{q=1}^{Q} \bm{W}^{{\bm{2}}}_{a_{q_{\left[\substack{\tau_1 \\\sigma_{11} \\\sigma_{21}}\right]}}} \bm{W}^{{\bm{2}}}_{b_{q_{\left[\substack{\tau_2 \\\sigma_{12} \\\sigma_{22}}\right]}}} x_{\left[\substack{t - \tau_1 \\s_1 - \sigma_{11} \\s_2 - \sigma_{21}}\right]} x_{\left[\substack{t - \tau_2 \\s_1 - \sigma_{12} \\s_2 - \sigma_{22}}\right]}\\
		& \nonumber = \sum_{\tau_1, \sigma_{11}, \sigma_{21}} \bm{W}^{\bm{1}}_{\left[\substack{\tau_1 \\\sigma_{11} \\\sigma_{21}}\right]} x_{\left[\substack{t - \tau_1 \\s_1 - \sigma_{11} \\s_2 - \sigma_{21}}\right]}\\ 
		& + \sum_{q=1}^{Q} \sum_{\substack{\tau_1, \sigma_{11}, \sigma_{21}\\ \tau_2, \sigma_{12}, \sigma_{22}}} \biggl(\bm{W}^{{\bm{2}}}_{a_{q_{\left[\substack{\tau_1 \\\sigma_{11} \\\sigma_{21}}\right]}}} x_{\left[\substack{t - \tau_1 \\s_1 - \sigma_{11} \\s_2 - \sigma_{21}}\right]} \biggr) \biggl( \bm{W}^{{\bm{2}}}_{b_{q_{\left[\substack{\tau_2 \\\sigma_{12} \\\sigma_{22}}\right]}}} x_{\left[\substack{t - \tau_2 \\s_1 - \sigma_{12} \\s_2 - \sigma_{22}}\right]}\biggr).
	\end{align}
\end{gather}
A larger $Q$ will provide a better approximation of the $2^{nd}$ order kernel.
The advantage of this class of approximation is two-fold. Firstly, the number of parameters can further be reduced, if for the $z^{th}$ layer, $(L_z \cdot [2p_{1_z}+1] \cdot [2p_{2_z}+1])^2 > 2Q(L_z \cdot [2p_{1_z}+1] \cdot [2p_{2_z}+1])$. A trade-off between performance and acceptable computational complexity must be accounted for when adopting such an approximation. Additionally, this makes it easier to implement the higher order kernels in Tensorflow \cite{abadi2016tensorflow} by using the built in convolutional operator.

The complexity of the approximate quadratic layers in the Cascaded Volterra Filter (see Figure \ref{bd_vf}) is reflected by the number of parameters as
\begin{gather}
	\begin{align}
		& \sum_{z=1}^{\mathcal{Z}}\biggl[ (L_z \cdot [2p_{1_z}+1] \cdot [2p_{2_z}+1]) + 2Q(L_z \cdot [2p_{1_z}+1] \cdot [2p_{2_z}+1]) \biggr].
	\end{align}
\end{gather}

\begin{prop}
	\label{qth_rank}
	The VNN second order kernel approximation in Equation \ref{approx} is a $Q^{th}$ rank approximation of the exact quadratic kernel $\bm{W^2}$.
\end{prop}
\begin{proof}
	For simplicity, consider a 1-D Volterra Filter with memory $L$. The quadratic weight matrix, $\bm{W^2}$ in such a case is of size $L \times L$, and Equation (\ref{approx}) becomes, $\bm{W^{2(Q)}}_{L \times L} = \sum_{q=1}^{Q} \bm{W^{2}}_{a_{q_{L \times 1}}} \bm{W^{2}}_{b_{q_{1 \times L}}}$.
	Consider the Singular Value Decomposition of the quadratic weight matrix, $\bm{W^2}$,
	\begin{equation}
		\bm{W^2} = \bm{U}\bm{\Sigma}\bm{V}^T,
		\label{svd}
	\end{equation}
	where, $\bm{U}$ and $\bm{V}$ are $L \times L$ matrices, and $\bm{\Sigma}$ is a diagonal matrix with singular values on the diagonal. Equation (\ref{svd}) can be re-written as,
	\begin{equation}
		\bm{W^2} = \sum_{q=1}^L u_q \sigma_q v_q^T,
	\end{equation} 
	where, $u_q$ and $v_q$ are the $q^{th}$ column of $\bm{U}$ and $\bm{V}$ respectively, and $\sigma_q$ is the $q^{th}$ diagonal element of $\bm{\Sigma}$. 
	A $Q^{th}$ rank approximation is then given as,
	\begin{gather}
		\begin{align}
			&\nonumber \bm{W^{2(Q)}} = \sum_{q=1}^Q u_q \sigma_q v_q^T\\
			& \quad \quad \quad = \sum_{q=1}^Q \hat{u}_q v_q^T,
		\end{align}
	\end{gather}
	where, $\hat{u}_q = u_q.\sigma_q$. If $\bm{W^2}_{a_q} = \hat{u}_q$ and $\bm{W^2}_{b_q} = v_q^T$,
	\begin{equation}
		\bm{W^{2(Q)}} = \sum_{q=1}^Q \bm{W^2}_{a_q} \bm{W^2}_{b_q},
	\end{equation}
	hence, confirming the approximation given in Equation \ref{approx} as a $Q^{th}$ rank approximation of the exact quadratic kernel.
\end{proof}
The matrices $\bm{W^2}_{a_q}$ and $\bm{W^2}_{b_q}$ are unknown beforehand, and will be learned as part of the training process driven by the classification performance of the system. 
We will evaluate and compare both approaches when implementing the second order kernel (i.e. approximation and exact method) in Section \ref{expt_res}. Figure \ref{approx_impl} illustrates the implementation of a $2^{nd}$ order filter using a $Q^{th}$ rank approximation.
\begin{figure}[h!]
	\centering
	\includegraphics[width=0.7\textwidth]{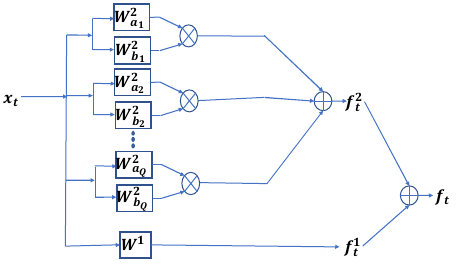}
	\caption{Implementation of a second order Volterra Filter using $Q^{th}$ rank approximation}
	\label{approx_impl}
\end{figure}

\subsection{Relation to LSTM Networks}
In this section we compare the proposed Volterra Neural Networks with LSTMs, which are frequently used for processing/analyzing temporal data.
An LSTM Network cell implemented as discussed in Section \ref{LSTM_rel_work}, can be shown to be a special case of the Volterra Filter.  

\paragraph*{CASE-1: No Activation Functions}
\begin{prop}
	The cell state of an LSTM network at time $t$, $C_{t}$ is a special case of a $2^{nd}$ order Volterra Filter, where the $2^{nd}$ order filter is approximated using $Q=1$ in Equation \ref{approx}, and weighed by the cell state at time $t-1$, $C_{t-1}$, i.e.
	\begin{equation}
		\label{C_t}
		C_t = C_{t-1}(\sum_{j=1}^2 \bm{W}_f^j s_j) + \sum_{j,k=1}^2 \bm{W}_{ic}^{jk} s_js_k,
	\end{equation}
	where, $\bm{s} = [h_{t-1}, x_t]$, and $\bm{W}_{ic} = \bm{W}_i \cdot \bm{W}_c$.
\end{prop}
\begin{proof}
	From Equations (\ref{forget_gate}), (\ref{input_gate}), (\ref{h_t}),
	At time t=0,
	\begin{equation}
		C_0 = h_0 = \bm{W}_0x_0.
	\end{equation}
	At time t=1,
	\begin{gather}
		\begin{align}
			&C_1 = (W_f^1h_0 + W_f^2x_1)(h_0) + (W_i^1h_0+ W_i^2x_1) (W_c^1h_0 + W_c^2x_1)\\
			& \quad = W_f^1h_0^2 + W_f^2x_1h_0 + W_i^1h_0.W_c^1h_0 + W_i^1h_0.W_c^2x_1 + W_i^2x_1.W_c^1h_0 + W_i^2x_1.W_c^2x_1.
		\end{align}
	\end{gather}
	Define the matrix $\bm{W}_{ic}$ such that, $\bm{W}_{ic} = \bm{W}_i \cdot \bm{W}_c$. 
	This leads to,
	\begin{gather}
		\begin{align}
			& C_1 = W_f^1h_0^2 + W_f^2x_1h_0 + W_{ic}^{11}h_0^2 + W_{ic}^{12}h_0x_1 + W_{ic}^{21}x_1h_0 + W_{ic}^{22}x_1^2\\
			& \quad = (W_f^1+W_{ic}^{11})h_0^2 + (W_f^2 + W_{ic}^{21}) x_1h_0 + W_{ic}^{21} h_0x_1 + W_{ic}^{22} x_1^2.
		\end{align}
	\end{gather}
	Note that the filter $\bm{W}_{ic}$ is equivalent to a $1^{st}$ rank approximation (Q=1) of the second order kernel of a Volterra Filter ($\bm{W}^2 = \sum_{q=1}^Q \bm{W}^{2}_{qa_{L \times 1}}\bm{W}^{2}_{qb_{1 \times L}}$). This means that the LSTM can be considered to be a special case of the Volterra Filter where a $1^{st}$ rank approximation of the $2^{nd}$ order kernel is used.
	
	At time t, 
	\begin{gather}
		\begin{align}
			& C_t = W_f^1h_{t-1}C_{t-1} + W_f^2x_tC_{t-1} + W_{ic}^{11} h_{t-1}^2 + W_{ic}^{12}h_{t-1}x_t + W_{ic}^{21}x_th_{t-1} + W_{ic}^{22}x_t^2 \\
			& \quad = C_{t-1}(\sum_{j=1}^2 \bm{W}_f^j s_j) + \sum_{j,k=1}^2 \bm{W}_{ic}^{jk} s_js_k,
		\end{align}
	\end{gather}
	where, $\bm{s} = [h_{t-1}, x_t]$, and $\bm{W}_{ic} = \bm{W}_i \cdot \bm{W}_c$. This formulation is equivalent to a $2^{nd}$ order Volterra Filter applied on the input $\bm{s}$, with the linear kernel $\bm{W}_f$ and quadratic kernel $\bm{W}_{ic}$.
\end{proof}
\begin{prop}
	The output $h_t$ of the LSTM network at time $t$, is a special case of the $3^{rd}$ order Volterra Filter, where the $2^{nd}$ and $3^{rd}$ order filters are approximated using $Q=1$ in Equation (\ref{approx}), and weighed by the cell state at time $t-1$, $C_{t-1}$, i.e.
	\begin{equation}
		h_t = C_{t-1}(\sum_{j,k=1}^2 \bm{W}^{jk}_{of} s_js_k) + \sum_{j,k,l=1}^2 \bm{W}^{jkl}_{oic} s_js_ks_l,
	\end{equation}
	where, $\bm{s} = [h_{t-1}, x_t]$, $\bm{W}_{of} = \bm{W}_o \cdot \bm{W}_f$, and $\bm{W}_{oic} = \bm{W}_o \cdot \bm{W}_i \cdot \bm{W}_c$.
\end{prop}
\begin{proof}
	Based on Equation (\ref{h_t}), the LSTM network output, $h_t$ can be written as,
	\begin{equation}
		h_t = (W_o^1h_{t-1} + W_o^2x_t)C_t.
	\end{equation} 
	Replacing $C_t$ by its expression in Equation (\ref{C_t}), and using $\bm{W}_{of} = \bm{W}_o \cdot \bm{W}_f$ and $\bm{W}_{oic} = \bm{W}_o \cdot \bm{W}_i \cdot \bm{W}_c$, 
	\begin{gather}
		\begin{align}
			&\nonumber h_t = W_{of}^{11} h_{t-1}^2 C_{t-1} + W_{of}^{12}h_{t-1}x_tC_{t-1} + W_{of}^{21}x_th_{t-1}C_{t-1} + W_{of}^{22} x_t^2C_{t-1} + W_{oic}^{111}h_{t-1}^3 + W_{oic}^{112}h_{t-1}^2x_t \\ 
			&  \quad \quad + W_{oic}^{121} h_{t-1}^2 + W_{oic}^{122}h_{t-1}x_t^2 + W_{oic}^{211}h_{t-1}^2x_t + W_{oic}^{212} h_{t-1}x_t^2 + W_{oic}^{221}h_{t-1}x_t^2 + W_{oic}^{222}x_t^3. 
		\end{align}
	\end{gather}
	The above expression can be re-written as,
	\begin{gather}
		\begin{align}
			&h_t = C_{t-1}(\sum_{j,k=1}^2 \bm{W}^{jk}_{of} s_js_k) + \sum_{j,k,l=1}^2 \bm{W}^{jkl}_{oic} s_js_ks_l,
		\end{align}
	\end{gather}
	where, $\bm{s} =[h_{t-1}, x_t]$. 
\end{proof}
On the other hand, for an O-VNN with $\mathcal{Z}$ layers,
\begin{gather}
	\begin{split}
		h_{t_{m_\mathcal{Z}}}^{\mathcal{Z}} = \sum_{j=0}^{L_{\mathcal{Z}}} W^{\mathcal{Z}1}_{j} h_{(m_\mathcal{Z}+L_\mathcal{Z})-j}^{\mathcal{Z}-1} + \sum_{j,k=0}^{L_\mathcal{Z}} W^{\mathcal{Z}2}_{j,k} \text{ }h_{(m_z+L_z)-j}^{\mathcal{Z}-1} h_{(m_z+L_z)-k}^{\mathcal{Z}-1},
	\end{split}
\end{gather}
where, $m_\mathcal{Z} \in [1:M_\mathcal{Z}]$, $M_\mathcal{Z} = M_{\mathcal{Z}-1}-L_{\mathcal{Z}}+1$, and $M_1 = t-L_1+1$.
In both scenarios, the system uses higher order relations between current and previous samples. In case of an LSTM network the cell state at $t-1$, $C_{t-1}$ is used in order to select features from previous frames that may be relevant to current frames. On the other hand the Volterra filter formulation explicitly selects the interactions between the frames and weighs them accordingly. 

\paragraph*{CASE-2: With Activation Functions}
The sigmoid activation function can be approximated as a taylor series,
\begin{equation}
	\sigma(x) = \frac{1}{1+e^{-x}} = \frac{1}{2} + \frac{1}{4}x - \frac{1}{48}x^3 + \frac{1}{480}x^5 -...
\end{equation}
This can be learned by a Volterra Network if required as it is a polynomial expression. Furthermore as long as the condition in Proposition \ref{convergence} is satisfied, the series is convergent.  
Similarly, a $\text{tanh}$ activation is also approximated by using a Taylor Series,
\begin{equation}
	\text{tanh}(x) = x - \frac{1}{3}x^3 + \frac{2}{15}x^5 - \frac{17}{315}x^7 + ...
\end{equation}
This means there is no longer a need to explicitly define the activation function, as the Volterra Neural Network will learn the required activation function as part of the learning/training process.

\section{Min-Norm Solution and Risk Analysis of a $2^{nd}$ Order Filter}
Given the vectorized $i^{th}$ data sample, $\bm{x}_{i_{1 \times d}} = \{x_i^1,x_i^2,...,x_i^d\}$, with its associated second order terms found by $\bm{x}_i^T\bm{x}_i$, provide a relevant data sample rewritten as,
\begin{equation}
	\bm{\hat{x}}_{i_{(d+d^2) \times 1}} = [\bm{x}_i, \{ vec(\bm{x}_i^T\bm{x}_i) \}].
\end{equation}
Consider $N$ such samples, i.e. $i \in \{1,...,N\}$, which will be used for training of the system, this leads to the data matrix, $\bm{\hat{X}}_{N \times (d+d^2)} = \{ \bm{\hat{x}}_i^T \}_{i \in \{ 1,...,N \}}$.

Now consider, $\bm{\hat{W}}_{(d+d^2) \times 1} = [\bm{W}^{1}_{1 \times d}, \bm{W}^{2}_{1 \times d^2}]^T$ where, $\bm{W}^1$ and $\bm{W}^2$ are the Volterra Filter weights.

If $\bm{Y}_{N \times 1}$ is the ground truth, we wish to find $\bm{\hat{W}}$ such that,
\begin{equation}
	\bm{\hat{X}}_{N \times (d+d^2) } \bm{\hat{W}}_{(d+d^2) \times 1} = \bm{Y}_{N \times 1}.
\end{equation}
The min-norm solution to the filter weights in this case turns out to be,
\begin{equation}
	\bm{\hat{W}} = \bm{\hat{X}}^T(\bm{\hat{X}}\bm{\hat{X}}^T)^{-1}\bm{Y},
\end{equation}
where, $\bm{\hat{X}}^T(\bm{\hat{X}}\bm{\hat{X}}^T)^{-1}$ is the pseudo-inverse of $\bm{\hat{X}}$, and $\bm{\hat{X}}\bm{\hat{X}}^T$ is assumed to have a full rank so that it is invertible.
\paragraph*{Risk Analysis}
The risk of the min-norm estimate, $\bm{\hat{W}}$ can be computed as discussed in \cite{bartlett2019benign},
\begin{gather}
	\begin{align}
		&R(\bm{\hat{W}})\text{ }= {\rm I\!E}_{\hat{x},y}(y-\hat{x}^T\bm{\hat{W}})^2 - {\rm I\!E}(y-\hat{x}^T\bm{\hat{W}}^\star)^2 \\
		& \quad \quad \quad = {\rm I\!E}_{\hat{x},y} (y - \hat{x}^T \bm{\hat{W}}^\star + \hat{x}^T(\bm{\hat{W}}^\star - \bm{\hat{W}}))^2 - {\rm I\!E}(y-\hat{x}^T\bm{\hat{W}}^\star)^2 \\
		&\quad \quad \quad = {\rm I\!E}_{\hat{x}} (\hat{x}^T (\bm{\hat{W}}^\star - \bm{\hat{W}}))^2,
	\end{align}
\end{gather}
where $\bm{\hat{W}^\star}$ is the least squares estimate, and ${\rm I\!E}$ is the expectation operator with respect to the observed data and the ground truth. Since $\bm{\hat{W}} = \bm{\hat{X}}^T(\bm{\hat{X}}\bm{\hat{X}}^T)^{-1}\bm{Y}$, 
\begin{equation}
	R(\bm{\hat{W}}) = {\rm I\!E}_{\hat{x}} {( \hat{x}^T(\bm{\hat{W}}^\star - \bm{\hat{X}}^T(\bm{\hat{X}}\bm{\hat{X}}^T)^{-1}\bm{Y}) )}
\end{equation}

Furthermore, considering distortion due to noise, $\epsilon = \bm{Y} - \bm{\hat{X}}\bm{\hat{W}}^\star$,

\begin{gather}
	\begin{align}
		& R(\bm{\hat{W}}) = {\rm I\!E}_{\hat{x}} ( \hat{x}^T(\bm{I} - \bm{\hat{X}}^T (\bm{\hat{X}}\bm{\hat{X}}^T)^{-1}))\bm{\hat{W}}^\star {- \hat{x}^T\bm{\hat{X}}^T (\bm{\hat{X}}\bm{\hat{X}}^T)^{-1}\epsilon) } ^2 \\
		& \quad \quad \quad  \leq 2 {\rm I\!E}_{\hat{x}} (\hat{x}^T(\bm{I} - \bm{\hat{X}}^T (\bm{\hat{X}}\bm{\hat{X}}^T)^{-1}))\bm{\hat{W}}^\star)^2 + 2 {\rm I\!E}_{\hat{x}} (\hat{x}^T\bm{\hat{X}}^T (\bm{\hat{X}}\bm{\hat{X}}^T)^{-1}\epsilon)^2\\
		& \quad \quad \quad = 2\bm{\hat{W^\star}}^T \bm{B} \bm{\hat{W^\star}} + 2\epsilon^T \bm{C} \epsilon, \label{risk}
	\end{align}
\end{gather}
where, $\bm{B} = (\bm{I} - \bm{\hat{X}}^T (\bm{\hat{X}}\bm{\hat{X}}^T)^{-1})\bm{\hat{\Sigma}}(\bm{I} - \bm{\hat{X}}^T (\bm{\hat{X}}\bm{\hat{X}}^T)^{-1})$, $\bm{\hat{\Sigma}} = \bm{\hat{x}}\bm{\hat{x}}^T$
and $\bm{C} = (\bm{\hat{X}}\bm{\hat{X}}^T)^{-1}\bm{\hat{X}} \bm{\hat{\Sigma}}\bm{\hat{X}}^T (\bm{\hat{X}}\bm{\hat{X}}^T)^{-1}.$

Equation (\ref{risk}) provides an upper bound for the excess risk when using the min-norm estimate $\bm{\hat{W}}$, in terms of the observed data, the optimal least-squares estimate $\bm{\hat{W}^\star}$ and the distortion due to noise, $\epsilon$, thus, showing the generalizability of this strategy.

\subsection{Two-Stream Volterra Networks}
\label{Two-Stream}
Most previous studies in action recognition in videos have noted the importance of using both the spatial and the temporal information for an improved recognition accuracy. As a result, we also propose the use of Volterra filtering in combining the two information streams, exploring a potential non-linear relationship between them. In Section \ref{expt_res} we will verify that this actually boosts the performance, thereby indicating some inherent non-linear relation between the two information streams. Independent Cascaded Volterra Filters are first used in order to extract features from each modality,
\begin{gather}
	\begin{align}
		\bm{F}_{[1:M_\mathcal{Z}]}^{\mathcal{Z}^{RGB}} = g_{\mathcal{Z}}^{RGB}(... g_2^{RGB}(g_1^{RGB}(\bm{X}^{RGB}_{[t-L+1:t]})))\\
		\bm{F}_{[1:M_\mathcal{Z}]}^{\mathcal{Z}^{OF}} = g_{\mathcal{Z}}^{OF}(... g_2^{OF}(g_1^{OF}(\bm{X}^{OF}_{[t-L+1:t]}))).
	\end{align}
\end{gather}
Upon gleaning features from the two streams, an additional Volterra Filter is solely used for combining the generated feature maps from both modalities,
\begin{gather}
	\begin{align}
		& \nonumber F_t^{(RGB+OF)} = \sum_{\tau_1,\sigma_{11}, \sigma_{21}, u_1} \bm{W^1}_{\left[ \substack{\tau_1\\ \sigma_{11} \\ \sigma_{21} \\ u_1} \right]} f_{\left[ \substack{M_\mathcal{Z}-\tau_1\\ s_1 - \sigma_{11}\\ s_2 - \sigma_{21}} \right]}^{\mathcal{Z}^{u_1}}\\
		& + \sum_{\substack{\tau_1,\sigma_{11}, \sigma_{21}, u_1\\ \tau_2,\sigma_{12}, \sigma_{22}, u_2}} \bm{W^2}_{\left[ \substack{\tau_1\\ \sigma_{11}\\ \sigma_{21}} \right] \left[ \substack{\tau_2\\  \sigma_{12}\\ \sigma_{22}} \right]} f_{\left[ \substack{M_\mathcal{Z}-\tau_1\\ s_1 - \sigma_{11}\\ s_2 - \sigma_{21}} \right]}^{\mathcal{Z}^{u_1}} f_{\left[ \substack{M_\mathcal{Z}-\tau_2\\ s_1 - \sigma_{12}\\ s_2 - \sigma_{22}} \right]}^{\mathcal{Z}^{u_2}},
	\end{align}
\end{gather}
where $\tau_j \in [0, L_{\mathcal{Z}+1}]$, $\sigma_{1j} \in [-p_1, p_1]$, $\sigma_{2j} \in [-p_2, p_2]$, and $u_j \in [RGB, OF]$.
Figure \ref{Fusion}-(c) shows the block diagram for fusing the two information streams. 

\section{Experiments and Results}
\label{expt_res}
\subsection{Action Recognition}
We proceed to evaluate the performance of this approach on three action recognition datasets, namely, Kinetics-400 \cite{carreira2017quo}, UCF-101 \cite{soomro2012ucf101} and HMDB-51 \cite{kuehne2011hmdb}. We present two versions of the VNN, a heavier complex version: O-VNN-H and a lighter, device friendly version: O-VNN-L. 
The performance comparison of the results with recent state of the art implementations on Kinetics-400 is presented in Table \ref{Kinetics-Perf}, and the comparisons on UCF-101 and HMDB-51 are presented in Tables \ref{acc_table_RGB} and \ref{acc_table_RGBOF}. 
Table \ref{acc_table_RGB} shows the comparison with techniques that only exploit the RGB stream, while Table \ref{acc_table_RGBOF} shows the comparison when both information streams are used. Note that our comparable performance to the state of the art is achieved with a significantly lower number of parameters (see Table \ref{Kinetics-Perf}). Furthermore, a significant boost in performance is achieved by allowing non-linear interaction between the two information streams. The Optical Flow is computed using the TV-L1 algorithm \cite{zach2007duality}. Note that we train the network from scratch on both datasets, and do not use a larger dataset for pre-training, in contrast to some of the previous implementations. 
The implementations that take advantage of a different dataset for pre-training are indicated by a \textit{`Y'} in the pre-training column, while those that do not, are indicated by \textit{`N'}. When training from scratch the proposed solution is able to achieve best performance for both scenarios: one stream networks (RGB frames only) and two-stream networks (RGB frames \& Optical Flow).
To fuse the two information streams (spatial and temporal), we evaluate the following techniques: 
\begin{enumerate}
	\item \textit{Decision Level Fusion (Figure \ref{Fusion}-(a)):} The decision probabilities $P_t^{RGB}(a_i)$ and $P_t^{OF}(a_i)$ are independently computed and are combined to determine the fused probability $P_t^f(a_i)$ using \textit{(a) Weighted Averaging: } $P_t^f(a_i) = \beta^{RGB} P_t^{RGB}(a_i) + \beta^{OF} P_t^{OF}(a_i)$, where $\beta^{RGB} + \beta^{OF} = 1$, which control the importance/contribution of the RGB and Optical Flow streams towards making a final decision, or \textit{(b) Event Driven Fusion \cite{roheda2018decision,roheda2019event}: } $P_t^f(a_i) = \gamma P_t^{\text{MAX MI}}(a_i^{RGB}, a_i^{OF}) + (1-\gamma) P_t^{\text{MIN MI}}(a_i^{RGB}, a_i^{OF})$, where $\gamma$ is a pseudo measure of correlation between the two information streams, $P_t^{\text{MAX MI}}(.)$ is the joint distribution with maximal mutual information, and $P_t^{\text{MIN MI}}(.)$ is the joint distribution with minimal mutual information.
	\item \textit{Feature Level Fusion:} Features are extracted from each stream independently, and are subsequently merged before making a decision. For this level of fusion we consider a simple \textit{Feature Concatenation} (see Figure \ref{Fusion}-(b)), and \textit{Two-Stream Volterra Filtering} (see Section \ref{Two-Stream}, Figure \ref{Fusion}-(c)). 
\end{enumerate}

\begin{figure*}[tbp] 
	\centering
	\includegraphics[width = 0.97\textwidth]{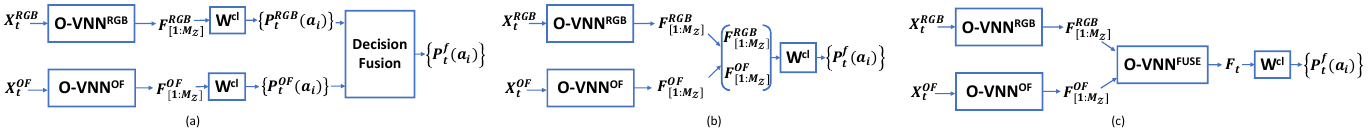}
	\caption{(a): Decision Level Fusion, (b): Feature Concatenation, (c): Two-Stream Volterra Filtering}
	\label{Fusion}
\end{figure*}

\begin{figure}[h]
	\centering
	\includegraphics[width=0.48\textwidth]{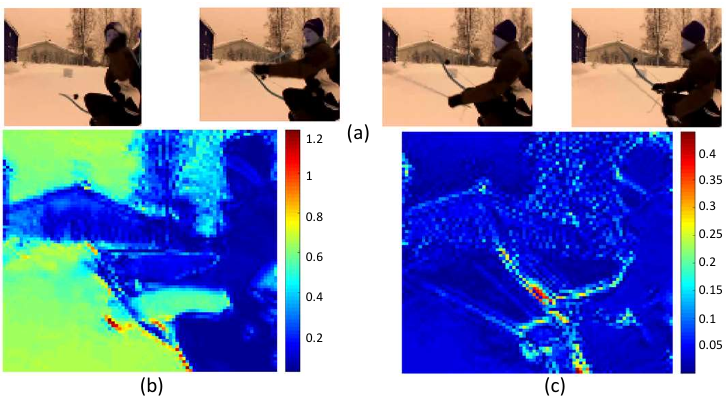}
	\caption{(a): Input Video, (b): Features extracted by only RGB stream, (c): Features extracted by Two-Stream Volterra Filtering}
	\label{Feats_0}
\end{figure}
\begin{figure}[h]
	\centering
	\includegraphics[width=0.48\textwidth]{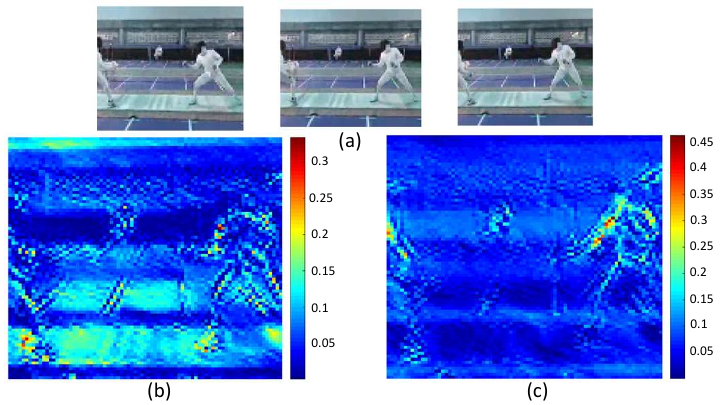}
	\caption{(a): Input Video, (b): Features extracted by only RGB stream, (c): Features extracted by Two-Stream Volterra Filtering}
	\label{Feats_1}
\end{figure}

\begin{table*}[tbp]
	\begin{center}
		\begin{tabular}{|M{7.8cm}|M{3.87cm}|M{2cm}| M{1.87cm}|}
			\hline
			\textbf{Method} & \textbf{Avg Accuracy Kinetics-400} & \textbf{Number of Parameters} & \textbf{GFLOPs} \\
			\hline
           	Temporal 3D ConvNets \cite{diba2017temporal} & 69.8 \% & 56M & 63.34 \\
			\hline
			SpatioTemporal Convolution \cite{tran2018closer} & 71.3 \% & 34M & 43.72 \\
			\hline
			Two-Stream Inflated 3D CNN \cite{carreira2017quo} & 74.2 \% & 25M & 37.43\\
			\hline
			Slowfast ResNet-50 \cite{feichtenhofer2019slowfast} & 77.0 \% & 34M & 50.58 \\
			\hline
			Slowfast ResNet-101 \cite{feichtenhofer2019slowfast} & \textbf{78.5 \%} & 62M & 96.79 \\
			\hline
			Two Stream O-VNN-L & 75.9 \% & \underline{\textbf{10M}} & \underline{\textbf{12.30}} \\
			\hline
			\textbf{Two Stream O-VNN-H} & \underline{\textbf{78.8} \%} & \textbf{26M} & \textbf{29.85} \\
			\hline
			
		\end{tabular}
	\end{center}
	\caption{Performance Evaluation for Two-Stream networks (RGB \& Optical Flow) on Kinetics-400}
	\label{Kinetics-Perf}
\end{table*}

\begin{table*} [tbp]
	\begin{center}	
		\begin{tabular}{|M{7.8cm}|M{4.0cm}|M{1.87cm}| M{1.87cm}|}
			\hline
			\textbf{Method} &\textbf{Pre-Training} & \textbf{Avg Accuracy UCF-101} & \textbf{Avg Accuracy HMDB-51}\\
			\hline
			Slow Fusion \cite{karpathy2014large} & \textit{Y} \small (Sports-1M) & 64.1 \% & -  \\
			\hline
			Deep Temporal Linear Encoding Networks \cite{diba2017deep} & \textit{Y} \small (Sports-1M) & 86.3 \% & 60.3 \% \\
			\hline
			Inflated 3D CNN \cite{carreira2017quo} & \textit{Y} \small (ImageNet + Kinetics) & \textbf{95.1 \%} & \textbf{74.3 \%} \\
			\hline
			O-VNN-H & \textit{Y} (Kinetics) & \textbf{95.3 \%} & \textbf{75.1 \%}\\
			\hline
			\hline
			Soomro et al, \citeyear{soomro2012ucf101} & \textit{N} & 43.9 \% & - \\ 
			\hline
			Single Frame CNN \cite{karpathy2014large,krizhevsky2012imagenet} & \textit{N} & 36.9 \% & - \\
			\hline
			Slow Fusion (\citeauthor{karpathy2014large,baccouche2011sequential,ji20123d}) & \textit{N} & 41.3 \% & - \\
			\hline
			3D-ConvNet \cite{carreira2017quo,tran2015learning} & \textit{N} & 51.6 \% & 24.3 \%\\
			\hline
			Volterra Filter & \textit{N} &  38.19 \% &  18.76 \% \\
			\hline
			O-VNN-H & \textit{N} &  \textbf{58.73 \%} &  \textbf{29.33 \%} \\
			\hline
			O-VNN-L & \textit{N} &  53.77 \% & 25.76 \% \\
			\hline 
		\end{tabular}
	\end{center}
	\caption{Performance Evaluation for one stream networks (RGB only): The proposed algorithm achieves best performance when trained from scratch}
	\label{acc_table_RGB}
	
\end{table*}

\begin{table*} [tbp]
	\begin{center}	
		\begin{tabular}{|M{7.3cm}|M{4.2cm}|M{1.87cm}| M{1.87cm}|}
			\hline
			\textbf{Method} &\textbf{Pre-Training} & \textbf{Avg Accuracy UCF-101} & \textbf{Avg Accuracy HMDB-51}\\
			\hline
			Two-Stream CNNs \cite{simonyan2014two} & \textit{Y} (ILSVRC-2012) & 88.0 \% & 72.7 \% \\
			\hline
			Deep Temporal Linear Encoding Networks \cite{diba2017deep} & \textit{Y} \small (BN-Inception + ImageNet) & 95.6 \% & 71.1 \% \\
			\hline
			Two Stream Inflated 3D CNN \cite{carreira2017quo} & \textit{Y} \small (ImageNet + Kinetics) & 98.0 \% & 80.9 \%  \\
			\hline
			Two-Stream O-VNN-H & \textit{Y} (Kinetics) & \textbf{98.49 \%} & \textbf{82.63 \%}\\
			\hline
			\hline
			Two Stream Inflated 3D CNN \cite{carreira2017quo} & \textit{N} & 88.8 \% & 62.2 \% \\
			\hline
			Weighted Averaging: O-VNN-H  & \textit{N} & 85.79 \% & 59.13 \% \\ 
			\hline
			Weighted Averaging: O-VNN-L  & \textit{N} & 81.53 \% & 55.67 \%\\
			\hline
			Event Driven Fusion: O-VNN-H & \textit{N} & 85.21 \% & 60.36 \%\\ 
			\hline
			Event Driven Fusion: O-VNN-L & \textit{N} & 80.37 \% & 57.89 \% \\ 
			\hline
			Feature Concatenation: O-VNN-H & \textit{N} & 82.31 \% & 55.88 \%\\ 
			\hline
			Feature Concatenation: O-VNN-L & \textit{N} & 78.79 \% & 51.08 \%\\ 
			\hline
			Two-Stream O-VNN-H & \textit{N} & \textbf{90.28 \%} & \textbf{ 65.61 \%}\\ 
			\hline
			Two-Stream O-VNN-L & \textit{N} & 86.16 \%  & 62.45 \% \\
			\hline
		\end{tabular}
	\end{center}
	\caption{Performance Evaluation for two stream networks (RGB \& Optical Flow): The proposed algorithm achieves best performance on both datasets.}
	\label{acc_table_RGBOF}
\end{table*}
The techniques are summarized in Figure \ref{Fusion}. The complex version of the proposed approach (O-VNN-H) comprises of 7 $2^{nd}$ order layers in both the RGB and the Optical stream. Each layer uses $L_z = 2$ and $p_{1_z}, p_{2_z} \in \{0,1,2\}$. The outputs of the RGB filter and the Optical Flow filter are then fed to a fusion layer with $L_{\text{Fuse}} = 2$ and $p_{1_{\text{Fuse}}}, p_{2_{\text{Fuse}}} \in \{0,1,2\}$. Similarly, a lighter version of the proposed method (O-VNN-L) is designed with 5 $2^{nd}$ order layers. 
It is clear from Table \ref{acc_table_RGBOF} that performing fusion using Volterra Filters significantly boosts the performance of the system. This shows that there does exist a non-linear relationship between the two modalities. This can also be confirmed from the fact that we can see significant values in the weights for the fusion layer (see Table \ref{Wt_Norm}).
\begin{table}[h]
	\begin{center}
		\begin{tabular}{|M{1.5cm}|M{1.5cm}|M{1.4cm}|M{2cm}|}
			\hline
			\small  & $u=RGB$ & $u=OF$ & $u=Fusion$ \\
			\hline
			$\frac{1}{2}{\left\| \bm{W}^{u} \right\|}_2^2$ & 352.15 & 241.2 & 341.3\\   
			\hline
		\end{tabular}
	\end{center}
	\caption{Norm of $\bm{W}^{u}$, where $u \in [RGB,OF,Fusion]$.}
	\label{Wt_Norm}
\end{table}
\begin{figure}[h]
	\centering
	\includegraphics[width=0.8\textwidth]{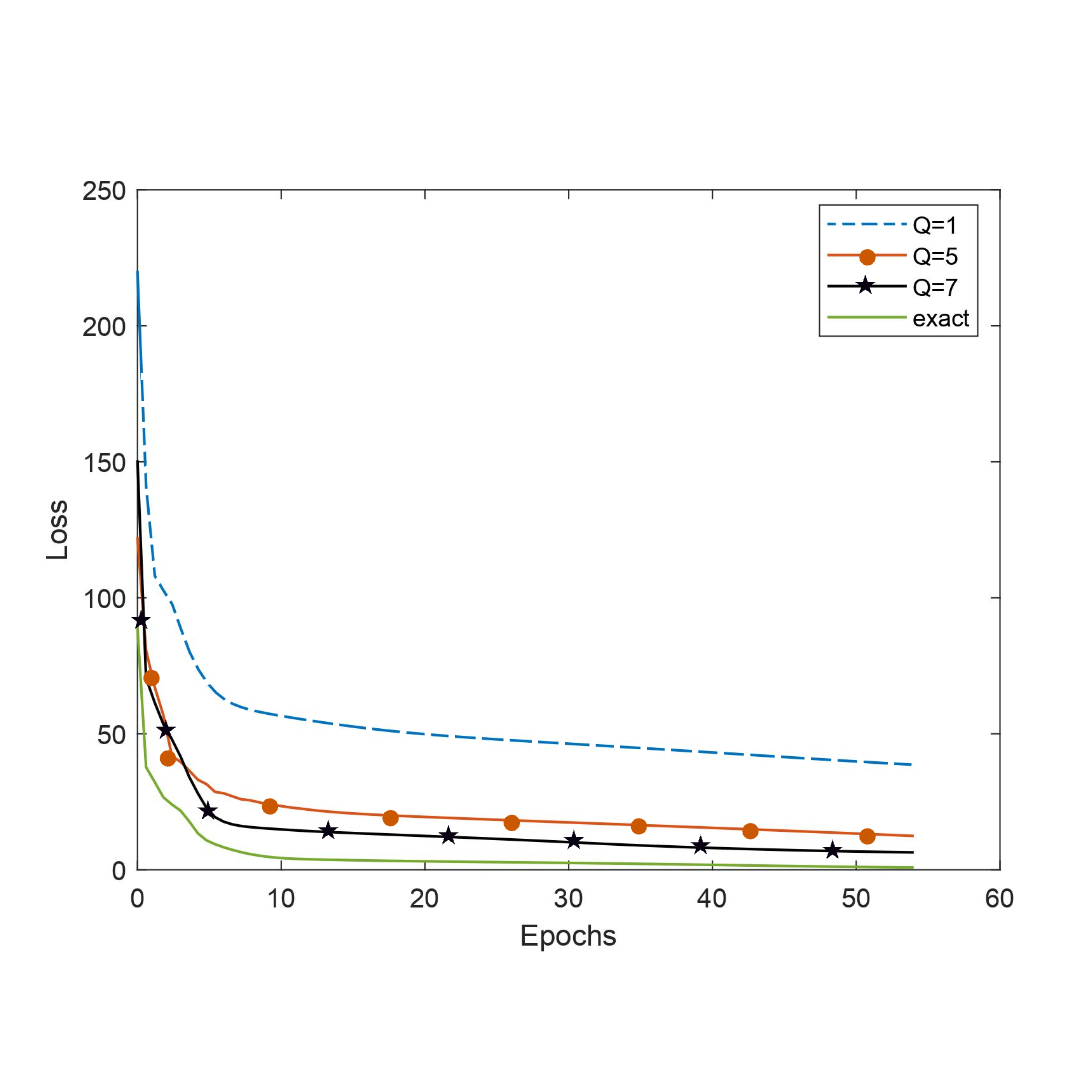}
	\caption{Epochs vs Loss for various number of multipliers for a Cascaded Volterra Filter}
	\label{Loss_Q}
\end{figure}   
Figures \ref{Feats_0} and  \ref{Feats_1} show one of the feature maps for an archery video and a fencing video. From Figures \ref{Feats_0}, \ref{Feats_1}-(b),(c) it can be seen that when only the RGB stream is used, a lot of the background area has high values, while on the other hand, when both streams are jointly used, the system is able to concentrate on more relevant features. In \ref{Feats_0}-(c), the system is seen to concentrate on the bow and arrow which are central to recognizing the action, while in \ref{Feats_1}-(c) the system is seen to concentrate on the pose of the human which is central to identifying a fencing action.
Figure \ref{Loss_Q} shows the Epochs vs Loss graph for a Cascaded Volterra Filter when a different number of multipliers ($Q$) are used to approximate the $2^{nd}$ order kernel. The green plot shows the loss when the exact kernel is learned, and it can be seen that the performance comes closer to the exact kernel as $Q$ is increased.

\begin{figure}[h]
	\centering
	\includegraphics[width=0.8\textwidth]{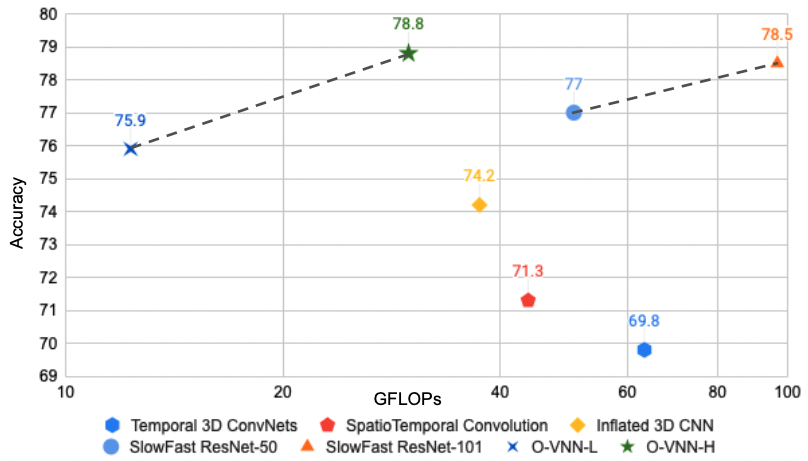}
	\caption{GFLOPs vs Accuracy for various Action Recognition methods discussed in Table \ref{Kinetics-Perf}}
	\label{GvA}
\end{figure}   

Figure \ref{GvA} depicts the plot for GFLOPs vs Accuracy of the various models compared in Table \ref{Kinetics-Perf} and highlights the reduction in model complexity achieved by the proposed method while achieving SOTA performance in action recognition. 

Figure \ref{SNR_UCF_101} evaluates the robustness of the proposed model (VNN) in presence of Gaussian noise and compares it with that of a CNN model \cite{carreira2017quo} for action recognition. It is observed that the VNN model is much more robust to Gaussian noise and provides up to ~20 \% improvement in classification accuracy. Furthermore,the VNN model experiences a graceful drop in performance compared to the CNN model which sees a severe drop at Signal to Noise Ratio (SNR) of 15 dB.

Figure \ref{FPS_UCF_101} evaluates performance when the number of Frames Per Second (FPS) is reduced. This illustrates that the VNN is better at modelling long term relationships in the time domain as compared to a CNN model. While unlike in the case of Gaussian noise, the VNN does see a significant drop in performance, it comes at a lower frame rate of 5 fps as compared to the 10 fps in the case of CNN.

\begin{figure}[h]
	\centering
	\includegraphics[width=0.7\textwidth]{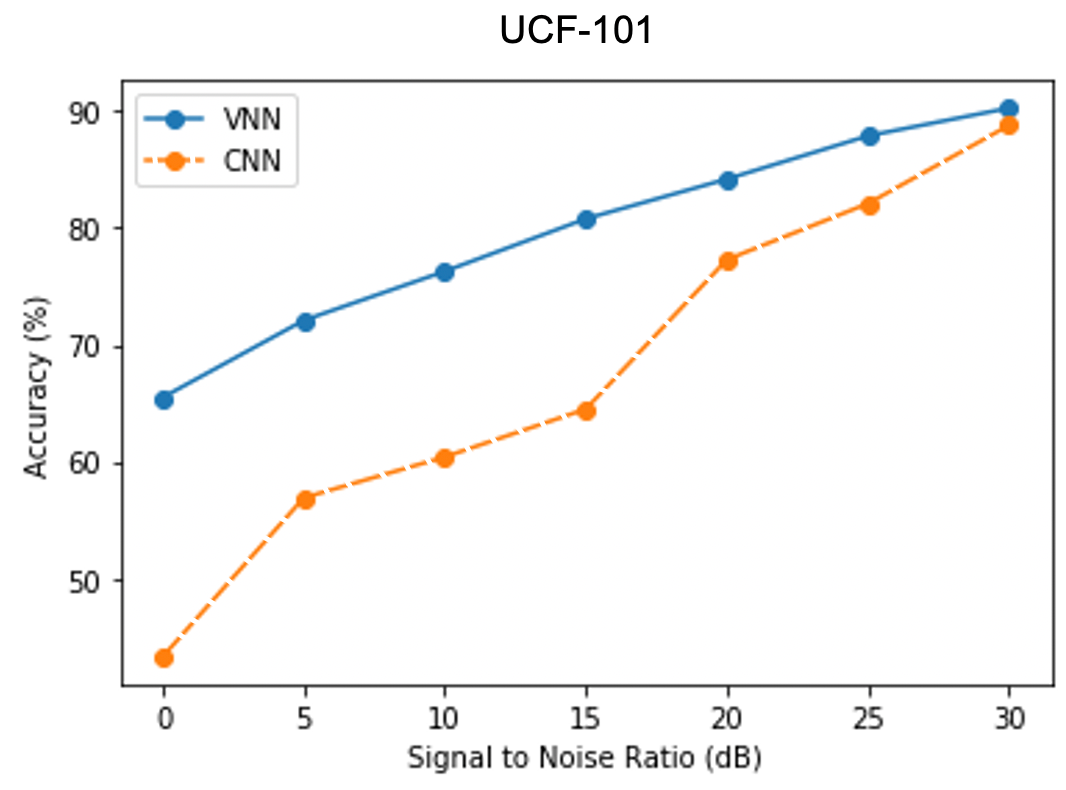}
	\caption{Performance comparison between VNN and CNN implementations when noise is added to the input videos}
	\label{SNR_UCF_101}
\end{figure}

\begin{figure}[h]
	\centering
	\includegraphics[width=0.7\textwidth]{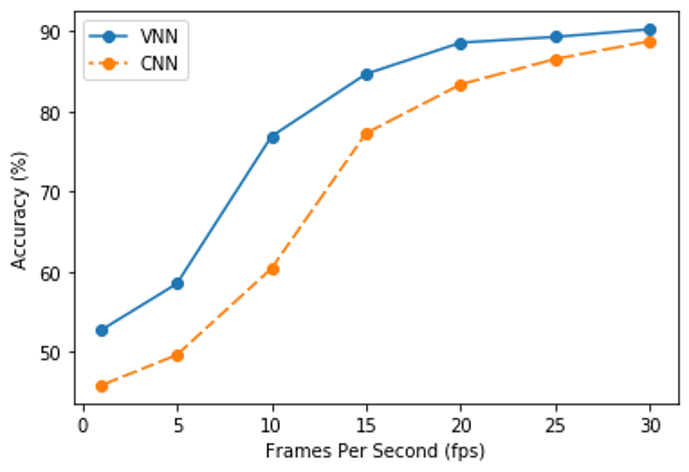}
	\caption{Performance comparison between VNN and CNN implementations for different fps}
	\label{FPS_UCF_101}
\end{figure}

\subsection{Image Generation}
To further demonstrate the capabilities of this VNN-based architecture we proceed to evaluate its generative capacity in a Generative Adversarial Netowrk (GAN) \cite{goodfellow2014generative} using our proposed formulation. Furthermore, we also use the VNN architecture jointly with a fixed dictionary as has been recently demonstrated in Stable GANs (STGANs) \cite{STGAN} to stabilize the training process. 

We design the experiment using the CIFAR10 dataset of  $60,000$ $32 \times 32$ color images of objects from 10 classes by allotting  50,000 images for training and 10,000 images for validation. To generate a $32 \times 32$ image, the input noise vector is first transformed into a latent space using a fixed  learned frame $\Theta^\star$. The generator uses two $2^{nd}$ order Volterra filter layers as opposed to four convolutional layers in \cite{STGAN}. The number of channels in the final layer is such that it is equal to the number of atoms in the fixed dictionary, which in this experiment is 384. The learned fixed dictionary $\Theta^\star$ is then multiplied by the generator output to produce generated image patches. The size of the dictionary in this experiment was $75 \times 384$. Figure \ref{generated_cifar} (a)-(c) shows the generated cifar10 images in this experiment and compares them with SPGAN \cite{shahinICCV} and Im-WGAN \cite{gulrajani2017improved}. The inception score of the proposed approach is compared with other existing approaches in Table \ref{IS} with a rather limited dictionary of Gaussians in one case (the performance would improve as the diversity of the dictionaries increases).

\begin{figure*}
	\centering
	\subfloat[Im-WGAN]{\includegraphics[width=0.5\textwidth]{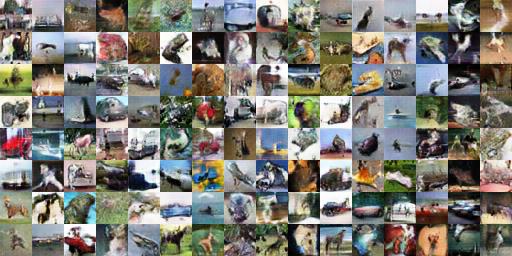}}\hfill
	\subfloat[SPGAN]{\includegraphics[width=0.5\textwidth]{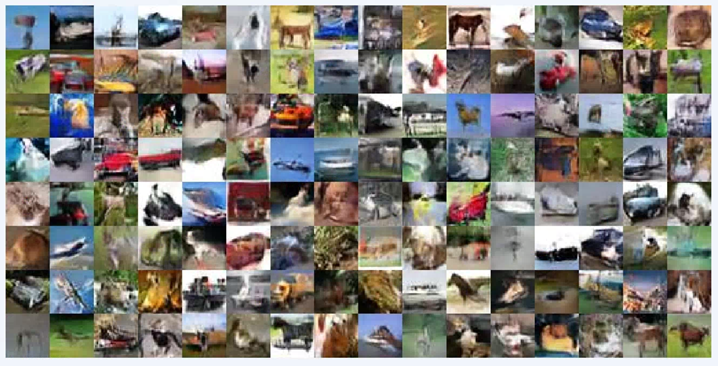}}\hfill
	\subfloat[Volterra STGAN]{\includegraphics[width=0.5\textwidth]{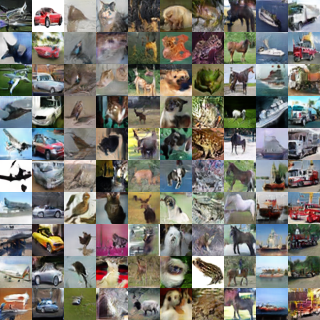}}\hfill
	\caption{Generated images using Cifar10 Dataset}
	\label{generated_cifar}
\end{figure*}

\begin{table}
	\begin{center}
		\begin{tabular}{|M{4.1cm}|M{1.5cm}|}
			\hline
			Method & \small Inception Score \\
			\hline
			\small WGAN & 5.95 \\ 
			\hline
			\small VNN-WGAN & 6.20 \\ 
			\hline
			\small SPGAN-recon & 6.70 \\
			\hline
			\small STGAN & 6.80 \\
			\hline
			\small Volterra STGAN &\textbf{7.05} \\ 
			\hline
		\end{tabular}
	\end{center}
	\caption{Comparison of various GAN based methods with proposed VNN STGAN in terms of inception score.} 
	\label{IS}
\end{table}

\section{Conclusion}
We proposed a novel network architecture for recognition of actions in videos, where the non-linearities were introduced by a Volterra Series Formulation. We propose a Cascaded Volterra Filter which leads to a significant reduction in parameters while exploring the same complexity of non-linearities in the data and attaining competitive and even better performance than SOA CNN. Such a Cascaded Volterra Filter was also shown to be a BIBO stable system. In addition, we also proposed the use of the Volterra Filter to fuse the spatial and temporal streams, hence leading to a non-linear fusion of the two streams. The combined competitive performance and  a significant reduction in parameter as well as sample complexity make VNN a viable alternative to SOA CNN.

\vskip 0.2in
\bibliography{sample}

\begin{thebibliography}{48}
\providecommand{\natexlab}[1]{#1}
\providecommand{\url}[1]{\texttt{#1}}
\expandafter\ifx\csname urlstyle\endcsname\relax
  \providecommand{\doi}[1]{doi: #1}\else
  \providecommand{\doi}{doi: \begingroup \urlstyle{rm}\Url}\fi

\bibitem[Abadi et~al.(2016)Abadi, Barham, Chen, Chen, Davis, Dean, Devin,
  Ghemawat, Irving, Isard, et~al.]{abadi2016tensorflow}
Mart{\'\i}n Abadi, Paul Barham, Jianmin Chen, Zhifeng Chen, Andy Davis, Jeffrey
  Dean, Matthieu Devin, Sanjay Ghemawat, Geoffrey Irving, Michael Isard, et~al.
\newblock Tensorflow: A system for large-scale machine learning.
\newblock In \emph{12th $\{$USENIX$\}$ Symposium on Operating Systems Design
  and Implementation ($\{$OSDI$\}$ 16)}, pages 265--283, 2016.

\bibitem[Baccouche et~al.(2011)Baccouche, Mamalet, Wolf, Garcia, and
  Baskurt]{baccouche2011sequential}
Moez Baccouche, Franck Mamalet, Christian Wolf, Christophe Garcia, and Atilla
  Baskurt.
\newblock Sequential deep learning for human action recognition.
\newblock In \emph{International workshop on human behavior understanding},
  pages 29--39. Springer, 2011.

\bibitem[Bartlett et~al.(2019)Bartlett, Long, Lugosi, and
  Tsigler]{bartlett2019benign}
Peter~L Bartlett, Philip~M Long, G{\'a}bor Lugosi, and Alexander Tsigler.
\newblock Benign overfitting in linear regression.
\newblock \emph{arXiv preprint arXiv:1906.11300}, 2019.

\bibitem[Carreira and Zisserman(2017)]{carreira2017quo}
Joao Carreira and Andrew Zisserman.
\newblock Quo vadis, action recognition? a new model and the kinetics dataset.
\newblock In \emph{proceedings of the IEEE Conference on Computer Vision and
  Pattern Recognition}, pages 6299--6308, 2017.

\bibitem[Dalal and Triggs(2005)]{dalal2005histograms}
Navneet Dalal and Bill Triggs.
\newblock Histograms of oriented gradients for human detection.
\newblock 2005.

\bibitem[Deng et~al.(2009)Deng, Dong, Socher, Li, Li, and
  Fei-Fei]{deng2009imagenet}
Jia Deng, Wei Dong, Richard Socher, Li-Jia Li, Kai Li, and Li~Fei-Fei.
\newblock Imagenet: A large-scale hierarchical image database.
\newblock In \emph{2009 IEEE conference on computer vision and pattern
  recognition}, pages 248--255. Ieee, 2009.

\bibitem[Diba et~al.(2017{\natexlab{a}})Diba, Fayyaz, Sharma, Karami, Arzani,
  Yousefzadeh, and Van~Gool]{diba2017temporal}
Ali Diba, Mohsen Fayyaz, Vivek Sharma, Amir~Hossein Karami, Mohammad~Mahdi
  Arzani, Rahman Yousefzadeh, and Luc Van~Gool.
\newblock Temporal 3d convnets: New architecture and transfer learning for
  video classification.
\newblock \emph{arXiv preprint arXiv:1711.08200}, 2017{\natexlab{a}}.

\bibitem[Diba et~al.(2017{\natexlab{b}})Diba, Sharma, and
  Van~Gool]{diba2017deep}
Ali Diba, Vivek Sharma, and Luc Van~Gool.
\newblock Deep temporal linear encoding networks.
\newblock In \emph{Proceedings of the IEEE conference on Computer Vision and
  Pattern Recognition}, pages 2329--2338, 2017{\natexlab{b}}.

\bibitem[Farabet et~al.(2012)Farabet, Couprie, Najman, and
  LeCun]{farabet2012learning}
Clement Farabet, Camille Couprie, Laurent Najman, and Yann LeCun.
\newblock Learning hierarchical features for scene labeling.
\newblock \emph{IEEE transactions on pattern analysis and machine
  intelligence}, 35\penalty0 (8):\penalty0 1915--1929, 2012.

\bibitem[Feichtenhofer et~al.(2016)Feichtenhofer, Pinz, and
  Zisserman]{feichtenhofer2016convolutional}
Christoph Feichtenhofer, Axel Pinz, and Andrew Zisserman.
\newblock Convolutional two-stream network fusion for video action recognition.
\newblock In \emph{Proceedings of the IEEE conference on computer vision and
  pattern recognition}, pages 1933--1941, 2016.

\bibitem[Feichtenhofer et~al.(2019)Feichtenhofer, Fan, Malik, and
  He]{feichtenhofer2019slowfast}
Christoph Feichtenhofer, Haoqi Fan, Jitendra Malik, and Kaiming He.
\newblock Slowfast networks for video recognition.
\newblock In \emph{Proceedings of the IEEE/CVF international conference on
  computer vision}, pages 6202--6211, 2019.

\bibitem[Firey(1960)]{firey1960remainder}
William~J Firey.
\newblock Remainder formulae in taylor's theorem.
\newblock \emph{The American Mathematical Monthly}, 67\penalty0 (9):\penalty0
  903--905, 1960.

\bibitem[Gao et~al.(2016)Gao, Beijbom, Zhang, and Darrell]{gao2016compact}
Yang Gao, Oscar Beijbom, Ning Zhang, and Trevor Darrell.
\newblock Compact bilinear pooling.
\newblock In \emph{Proceedings of the IEEE conference on computer vision and
  pattern recognition}, pages 317--326, 2016.

\bibitem[Goodfellow et~al.(2014)Goodfellow, Pouget-Abadie, Mirza, Xu,
  Warde-Farley, Ozair, Courville, and Bengio]{goodfellow2014generative}
Ian Goodfellow, Jean Pouget-Abadie, Mehdi Mirza, Bing Xu, David Warde-Farley,
  Sherjil Ozair, Aaron Courville, and Yoshua Bengio.
\newblock Generative adversarial nets.
\newblock In \emph{Advances in neural information processing systems}, pages
  2672--2680, 2014.

\bibitem[Gulrajani et~al.(2017)Gulrajani, Ahmed, Arjovsky, Dumoulin, and
  Courville]{gulrajani2017improved}
Ishaan Gulrajani, Faruk Ahmed, Martin Arjovsky, Vincent Dumoulin, and Aaron~C
  Courville.
\newblock Improved training of wasserstein gans.
\newblock In \emph{Advances in neural information processing systems}, pages
  5767--5777, 2017.

\bibitem[Hoffman et~al.(2016)Hoffman, Gupta, and Darrell]{hoffman2016learning}
Judy Hoffman, Saurabh Gupta, and Trevor Darrell.
\newblock Learning with side information through modality hallucination.
\newblock In \emph{Proceedings of the IEEE Conference on Computer Vision and
  Pattern Recognition}, pages 826--834, 2016.

\bibitem[Isola et~al.(2017)Isola, Zhu, Zhou, and Efros]{isola2017image}
Phillip Isola, Jun-Yan Zhu, Tinghui Zhou, and Alexei~A Efros.
\newblock Image-to-image translation with conditional adversarial networks.
\newblock In \emph{Proceedings of the IEEE conference on computer vision and
  pattern recognition}, pages 1125--1134, 2017.

\bibitem[Ji et~al.(2012)Ji, Xu, Yang, and Yu]{ji20123d}
Shuiwang Ji, Wei Xu, Ming Yang, and Kai Yu.
\newblock 3d convolutional neural networks for human action recognition.
\newblock \emph{IEEE transactions on pattern analysis and machine
  intelligence}, 35\penalty0 (1):\penalty0 221--231, 2012.

\bibitem[Karpathy et~al.(2014)Karpathy, Toderici, Shetty, Leung, Sukthankar,
  and Fei-Fei]{karpathy2014large}
Andrej Karpathy, George Toderici, Sanketh Shetty, Thomas Leung, Rahul
  Sukthankar, and Li~Fei-Fei.
\newblock Large-scale video classification with convolutional neural networks.
\newblock In \emph{Proceedings of the IEEE conference on Computer Vision and
  Pattern Recognition}, pages 1725--1732, 2014.

\bibitem[Kay et~al.(2017)Kay, Carreira, Simonyan, Zhang, Hillier,
  Vijayanarasimhan, Viola, Green, Back, Natsev, et~al.]{kay2017kinetics}
Will Kay, Joao Carreira, Karen Simonyan, Brian Zhang, Chloe Hillier, Sudheendra
  Vijayanarasimhan, Fabio Viola, Tim Green, Trevor Back, Paul Natsev, et~al.
\newblock The kinetics human action video dataset.
\newblock \emph{arXiv preprint arXiv:1705.06950}, 2017.

\bibitem[Kong and Fu(2018)]{kong2018human}
Yu~Kong and Yun Fu.
\newblock Human action recognition and prediction: A survey.
\newblock \emph{arXiv preprint arXiv:1806.11230}, 2018.

\bibitem[Krim et~al.()Krim, Jaffard, Roheda, Mahdizadehaghdam, and
  Panahi]{STGAN}
Hamid Krim, S~Jaffard, Siddharth Roheda, Shahin Mahdizadehaghdam, and Ashkan
  Panahi.
\newblock On stabilizing generative adversarial networks (stgans).

\bibitem[Krizhevsky et~al.(2012)Krizhevsky, Sutskever, and
  Hinton]{krizhevsky2012imagenet}
Alex Krizhevsky, Ilya Sutskever, and Geoffrey~E Hinton.
\newblock Imagenet classification with deep convolutional neural networks.
\newblock In \emph{Advances in neural information processing systems}, pages
  1097--1105, 2012.

\bibitem[Kuehne et~al.(2011)Kuehne, Jhuang, Garrote, Poggio, and
  Serre]{kuehne2011hmdb}
Hildegard Kuehne, Hueihan Jhuang, Est{\'\i}baliz Garrote, Tomaso Poggio, and
  Thomas Serre.
\newblock Hmdb: a large video database for human motion recognition.
\newblock In \emph{2011 International Conference on Computer Vision}, pages
  2556--2563. IEEE, 2011.

\bibitem[Kumar et~al.(2011)Kumar, Banerjee, Vemuri, and
  Pfister]{kumar2011trainable}
Ritwik Kumar, Arunava Banerjee, Baba~C Vemuri, and Hanspeter Pfister.
\newblock Trainable convolution filters and their application to face
  recognition.
\newblock \emph{IEEE transactions on pattern analysis and machine
  intelligence}, 34\penalty0 (7):\penalty0 1423--1436, 2011.

\bibitem[LeCun et~al.(1998)LeCun, Bottou, Bengio, Haffner,
  et~al.]{lecun1998gradient}
Yann LeCun, L{\'e}on Bottou, Yoshua Bengio, Patrick Haffner, et~al.
\newblock Gradient-based learning applied to document recognition.
\newblock \emph{Proceedings of the IEEE}, 86\penalty0 (11):\penalty0
  2278--2324, 1998.

\bibitem[Lin et~al.(2015)Lin, RoyChowdhury, and Maji]{lin2015bilinear}
Tsung-Yu Lin, Aruni RoyChowdhury, and Subhransu Maji.
\newblock Bilinear cnns for fine-grained visual recognition.
\newblock \emph{arXiv preprint arXiv:1504.07889}, 2015.

\bibitem[Liu et~al.(2009)Liu, Luo, and Shah]{liu2009recognizing}
Jingen Liu, Jiebo Luo, and Mubarak Shah.
\newblock Recognizing realistic actions from videos in the wild.
\newblock Citeseer, 2009.

\bibitem[Mahdizadehaghdam et~al.(2019)Mahdizadehaghdam, Panahi, and
  Krim]{shahinICCV}
Shahin Mahdizadehaghdam, Ashkan Panahi, and Hamid Krim.
\newblock Sparse generative adversarial network.
\newblock In \emph{Proceedings of the IEEE International Conference on Computer
  Vision Workshops}, pages 0--0, 2019.

\bibitem[Niebles et~al.(2010)Niebles, Chen, and Fei-Fei]{niebles2010modeling}
Juan~Carlos Niebles, Chih-Wei Chen, and Li~Fei-Fei.
\newblock Modeling temporal structure of decomposable motion segments for
  activity classification.
\newblock In \emph{European conference on computer vision}, pages 392--405.
  Springer, 2010.

\bibitem[Osowski and Quang(1994)]{osowski1994multilayer}
Stanislaw Osowski and T~Vu Quang.
\newblock Multilayer neural network structure as volterra filter.
\newblock In \emph{Proceedings of IEEE International Symposium on Circuits and
  Systems-ISCAS'94}, volume~6, pages 253--256. IEEE, 1994.

\bibitem[Roheda et~al.(2018{\natexlab{a}})Roheda, Krim, Luo, and
  Wu]{roheda2018decision}
Siddharth Roheda, Hamid Krim, Zhi-Quan Luo, and Tianfu Wu.
\newblock Decision level fusion: An event driven approach.
\newblock In \emph{2018 26th European Signal Processing Conference (EUSIPCO)},
  pages 2598--2602. IEEE, 2018{\natexlab{a}}.

\bibitem[Roheda et~al.(2018{\natexlab{b}})Roheda, Riggan, Krim, and
  Dai]{roheda2018cross}
Siddharth Roheda, Benjamin~S Riggan, Hamid Krim, and Liyi Dai.
\newblock Cross-modality distillation: A case for conditional generative
  adversarial networks.
\newblock In \emph{2018 IEEE International Conference on Acoustics, Speech and
  Signal Processing (ICASSP)}, pages 2926--2930. IEEE, 2018{\natexlab{b}}.

\bibitem[Roheda et~al.(2019)Roheda, Krim, Luo, and Wu]{roheda2019event}
Siddharth Roheda, Hamid Krim, Zhi-Quan Luo, and Tianfu Wu.
\newblock Event driven fusion.
\newblock \emph{arXiv preprint arXiv:1904.11520}, 2019.

\bibitem[Rudin et~al.(1964)]{rudin1964principles}
Walter Rudin et~al.
\newblock \emph{Principles of mathematical analysis}, volume~3.
\newblock McGraw-hill New York, 1964.

\bibitem[Schetzen(1980)]{schetzen1980volterra}
Martin Schetzen.
\newblock The volterra and wiener theories of nonlinear systems.
\newblock 1980.

\bibitem[Sermanet et~al.(2013)Sermanet, Eigen, Zhang, Mathieu, Fergus, and
  LeCun]{sermanet2013overfeat}
Pierre Sermanet, David Eigen, Xiang Zhang, Micha{\"e}l Mathieu, Rob Fergus, and
  Yann LeCun.
\newblock Overfeat: Integrated recognition, localization and detection using
  convolutional networks.
\newblock \emph{arXiv preprint arXiv:1312.6229}, 2013.

\bibitem[Simonyan and Zisserman(2014)]{simonyan2014two}
Karen Simonyan and Andrew Zisserman.
\newblock Two-stream convolutional networks for action recognition in videos.
\newblock In \emph{Advances in neural information processing systems}, pages
  568--576, 2014.

\bibitem[Sivic and Zisserman(2003)]{sivic2003video}
Josef Sivic and Andrew Zisserman.
\newblock Video google: A text retrieval approach to object matching in videos.
\newblock In \emph{null}, page 1470. IEEE, 2003.

\bibitem[Soomro et~al.(2012)Soomro, Zamir, and Shah]{soomro2012ucf101}
Khurram Soomro, Amir~Roshan Zamir, and Mubarak Shah.
\newblock Ucf101: A dataset of 101 human actions classes from videos in the
  wild.
\newblock \emph{arXiv preprint arXiv:1212.0402}, 2012.

\bibitem[Stone(1948)]{stone1948generalized}
Marshall~H Stone.
\newblock The generalized weierstrass approximation theorem.
\newblock \emph{Mathematics Magazine}, 21\penalty0 (5):\penalty0 237--254,
  1948.

\bibitem[Tran et~al.(2015)Tran, Bourdev, Fergus, Torresani, and
  Paluri]{tran2015learning}
Du~Tran, Lubomir Bourdev, Rob Fergus, Lorenzo Torresani, and Manohar Paluri.
\newblock Learning spatiotemporal features with 3d convolutional networks.
\newblock In \emph{Proceedings of the IEEE international conference on computer
  vision}, pages 4489--4497, 2015.

\bibitem[Tran et~al.(2018)Tran, Wang, Torresani, Ray, LeCun, and
  Paluri]{tran2018closer}
Du~Tran, Heng Wang, Lorenzo Torresani, Jamie Ray, Yann LeCun, and Manohar
  Paluri.
\newblock A closer look at spatiotemporal convolutions for action recognition.
\newblock In \emph{Proceedings of the IEEE conference on Computer Vision and
  Pattern Recognition}, pages 6450--6459, 2018.

\bibitem[Volterra(2005)]{volterra2005theory}
Vito Volterra.
\newblock \emph{Theory of functionals and of integral and integro-differential
  equations}.
\newblock Courier Corporation, 2005.

\bibitem[Wang et~al.(2009)Wang, Ullah, Klaser, Laptev, and
  Schmid]{wang2009evaluation}
Heng Wang, Muhammad~Muneeb Ullah, Alexander Klaser, Ivan Laptev, and Cordelia
  Schmid.
\newblock Evaluation of local spatio-temporal features for action recognition.
\newblock 2009.

\bibitem[Yi et~al.(2011)Yi, Krim, and Norris]{yi2011human}
Sheng Yi, Hamid Krim, and Larry~K Norris.
\newblock Human activity modeling as brownian motion on shape manifold.
\newblock In \emph{International Conference on Scale Space and Variational
  Methods in Computer Vision}, pages 628--639. Springer, 2011.

\bibitem[Zach et~al.(2007)Zach, Pock, and Bischof]{zach2007duality}
Christopher Zach, Thomas Pock, and Horst Bischof.
\newblock A duality based approach for realtime tv-l 1 optical flow.
\newblock In \emph{Joint pattern recognition symposium}, pages 214--223.
  Springer, 2007.

\bibitem[Zoumpourlis et~al.(2017)Zoumpourlis, Doumanoglou, Vretos, and
  Daras]{zoumpourlis2017non}
Georgios Zoumpourlis, Alexandros Doumanoglou, Nicholas Vretos, and Petros
  Daras.
\newblock Non-linear convolution filters for cnn-based learning.
\newblock In \emph{Proceedings of the IEEE International Conference on Computer
  Vision}, pages 4761--4769, 2017.

\end{thebibliography}

\end{document}